\pdfoutput=1

\documentclass[11pt]{article}
\usepackage{fullpage}
\usepackage{mathtools,amsfonts,amsthm,amssymb,bbm}
\usepackage{xcolor}
\usepackage[backref,colorlinks,citecolor=blue,linkcolor=magenta,bookmarks=true]{hyperref}
\usepackage[nameinlink]{cleveref}
\usepackage{algorithm,algorithmic}

\title{Collaborative Learning with Different Labeling Functions}
\author{
Yuyang Deng\thanks{Pennsylvania State University. Email: \texttt{yzd82@psu.edu}.}
\and
Mingda Qiao\thanks{University of California, Berkeley. Email: \texttt{mingda.qiao@berkeley.edu}. Part of this work was done while the author was a graduate student at Stanford University.}
}
\date{}

\usepackage{epsf}
\usepackage{graphics}
\usepackage{wrapfig}
\usepackage{psfrag}

\usepackage{color}

\usepackage{mathtools}
\usepackage{amsfonts}
\usepackage{amsthm}
\usepackage{amsmath}
\usepackage{amssymb}
\newtheorem{theorem}{Theorem}
\newtheorem{lemma}{Lemma}
\newtheorem{definition}[lemma]{Definition}
\newtheorem{remark}[lemma]{Remark}

\newtheorem{claim}{Claim}
\long\def\comment#1{}

\newcommand{\1}{\ensuremath{{\sf (i)}}}

\newcommand{\R}{\mathbb{R}}
\newcommand{\N}{\mathbb{N}}
\newcommand{\cA}{\mathcal{A}}

\newcommand{\cD}{\mathcal{D}}  % distribution
\newcommand{\cX}{\mathcal{X}}  % set of features
\DeclareMathOperator*{\argmin}{arg\,min}

\newtheorem{example}[theorem]{Example}
\newtheorem{problem}{Problem}

\renewcommand{\1}[1]{\mathbbm{1}\left\{#1\right\}} % Indicator
\newcommand{\A}{\mathcal{A}}
\newcommand{\Bern}{\mathsf{Bernoulli}}
\newcommand{\Bin}{\mathsf{Binomial}} % Binomial distribution

\newcommand{\Ex}[2]{\operatorname*{\mathbb{E}}_{#1}\left[#2\right]}
\newcommand{\D}{\mathcal{D}} % Distribution
\newcommand{\Dact}{\D^{\mathsf{actual}}}
\newcommand{\Dhard}{\D^{\mathsf{hard}}}
\newcommand{\eps}{\epsilon}
\newcommand{\err}{\mathrm{err}}
\newcommand{\F}{\mathcal{F}} % Function class
\newcommand{\goodevent}{\mathcal{E}^{\mathsf{good}}}
\newcommand{\NP}{\mathsf{NP}}
\newcommand{\poly}{\mathrm{poly}}
\newcommand{\polylog}{\mathrm{polylog}}
\newcommand{\pr}[2]{\Pr_{#1}\left[#2\right]}
\newcommand{\RP}{\mathsf{RP}}
\newcommand{\X}{\mathcal{X}} % Instance space

\newcommand{\alglinelabel}{%
  \addtocounter{ALC@line}{-1}% Reduce line counter by 1
  \refstepcounter{ALC@line}% Increment line counter with reference capability
  \label% Regular \label
}

\begin{document}

\maketitle

\begin{abstract}
    We study a variant of Collaborative PAC Learning, in which we aim to learn an accurate classifier for each of the $n$ data distributions, while minimizing the number of samples drawn from them in total. Unlike in the usual collaborative learning setup, it is not assumed that there exists a single classifier that is simultaneously accurate for all distributions.

    We show that, when the data distributions satisfy a weaker realizability assumption, which appeared in~\cite{CM12} in the context of multi-task learning, sample-efficient learning is still feasible. We give a learning algorithm based on Empirical Risk Minimization (ERM) on a natural augmentation of the hypothesis class, and the analysis relies on an upper bound on the VC dimension of this augmented class.

    In terms of the computational efficiency, we show that ERM on the augmented hypothesis class is $\NP$-hard, which gives evidence against the existence of computationally efficient learners in general. On the positive side, for two special cases, we give learners that are both sample- and computationally-efficient.
\end{abstract}

\section{Introduction}

In recent years, the remarkable success of data-driven machine learning has transformed numerous domains using the vast and diverse datasets collected from the real world. An ever-increasing volume of decentralized data is generated on a multitude of distributed devices, such as smartphones and personal computers. To better utilize these distributed data shards, we are faced with a challenge: how to effectively learn from these heterogeneous and noisy data sources?

{\em Collaborative PAC Learning}~\cite{BHPQ17} is a theoretical framework that abstracts the challenge above. In this model, there are $n$ data distributions $\D_1, \D_2, \ldots, \D_n$, from which we can adaptively sample. We are asked to learn $n$ classifiers $\hat{f}_1, \hat f_2, \ldots, \hat{f}_n$, such that each $\hat f_i$ has an error at most $\eps$ on $\D_i$. The goal is to minimize the number of labeled examples that we sample from the $n$ distributions in total.

Note that if we ignore the potential connection among the $n$ learning tasks and solve them separately, the sample complexity is necessarily linear in $n$. Previously, \cite{BHPQ17} introduced a sample-efficient algorithm when all distributions admit the same labeling function, i.e., some classifier in the hypothesis class has a zero error on every $\D_i$. Their algorithm has an $O((d + n)\log n)$ sample complexity, where $d$ is the VC dimension of hypothesis class.\footnote{For brevity, we treat the accuracy and confidence parameters as constants here.} When $d$ is large, the \emph{overhead} of the sample complexity is significantly reduced from $n$ to $\log n$.

However, in the real world, it is often too strong an assumption that every data distribution is consistent with the \emph{same} ground truth classifier. This is especially true when we are learning for a diverse population consisting of multiple sub-groups, each with different demographics and preferences. In light of this, we study a model of {\em collaborative learning with different labeling functions}. In particular, we aim to determine the conditions under which sample-efficient learning is viable when the data from different sources are labeled differently, and find the optimal sample complexity.

The contribution of this work is summarized as follows; see Section~\ref{sec:our-results} for formal statements of our results.
\begin{itemize}
    \item We formalize a model of collaborative learning with different labeling functions, and a sufficient condition, termed \emph{$(k,\epsilon)$-realizability}, for sample-efficient collaborative learning. This realizability assumption was used by~\cite{CM12} in the context of multi-task learning. Under this assumption, we give a learning algorithm with sample complexity $O(kd\log(n/k) + n\log n)$. This algorithm is based on Empirical Risk Minimizarion (ERM) over an augmentation of the hypothesis class.

    \item We show that the ERM problem over the augmented hypothesis class is always $\NP$-hard when $k \ge 3$, and $\NP$-hard for a specific hypothesis class when $k = 2$. This rules out efficient learners based on ERM, as well as \emph{strongly proper} learners that always output at most $k$ different classifiers in the hypothesis class.

    \item Finally, we identify two cases in which computationally efficient learning \emph{is} possible. When all distributions share the same marginal distribution on $\cX$, we give a simple polynomial-time algorithm that matches the information-theoretic bound. When the hypothesis class satisfies a ``$2$-refutability'' assumption, we give a different algorithm based on approximate graph coloring, which outperforms the na\"ive approach with an $\Omega(nd)$ sample complexity.
\end{itemize}

\subsection{Problem Setup}
We adopt the following standard model of binary classification: The hypothesis class $\F \subseteq \{0, 1\}^{\X}$ is a family of binary functions over the instance space $\X$. A data distribution $\D$ is a distribution over $\X \times \{0, 1\}$. The \emph{population error} of a function $f: \X \to \{0, 1\}$ on data distribution $\D$ is defined as
\[
    \err_{\D}(f)
\coloneqq
    \pr{(x, y) \sim \D}{f(x) \ne y}.
\]
A dataset is a multiset with elements in $\X \times \{0, 1\}$. The \emph{training error} of $f:\X \to \{0, 1\}$ on dataset $S = \{(x_i, y_i)\}_{i \in [m]}$ is defined as
\[
    \err_S(f) \coloneqq \frac{1}{m}\sum_{i=1}^{m}\1{f(x_i) \ne y_i}.
\]

The learning algorithm is given sample access to $n$ data distributions $\D_1, \D_2, \ldots, \D_n$. At each step, the algorithm is allowed to choose one of the $n$ distributions (possibly depending on the previous samples) and draw a labeled example from it. The algorithm may terminate at any time and return $n$ functions $\hat f_1, \hat f_2, \ldots, \hat f_n$. The learning algorithm is $(\eps, \delta)$-PAC if it satisfies
\[
    \pr{}{\err_{\D_i}(\hat f_i) \le \eps,~\forall i \in [n]} \ge 1 - \delta.
\]
The sample complexity of the algorithm is the expected number of labeled examples sampled in total.

Note that the above is almost the same as the personalized setup (i.e., the algorithm may output different classifiers for different distributions) of the model of~\cite{BHPQ17}, except that in their model, in addition, it is assumed that there exists a classifier in $\F$ with a zero error on every $\D_i$.

\subsection{Our Results}\label{sec:our-results}
\paragraph{A sufficient condition for sample-efficient learning.} We start by stating a sufficient condition for $n$ distributions to be learnable with a sample complexity that is (almost) linear in some parameter $k$ instead of in $n$.

\begin{definition}[$(k, \eps)$-Realizability]\label{def:k-eps-realizable}
    Distributions $\D_1$, $\D_2$, $\ldots$, $\D_n$ are $(k, \eps)$-realizable with respect to hypothesis class $\F$, if there exist $f^*_1, f^*_2, \ldots, f^*_k \in \F$ such that $\min_{j \in [k]}\err_{\D_i}(f^*_j) \le \eps$ holds for every $i \in [n]$.
\end{definition}
In words, $(k, \eps)$-realizability states that we can find $k$ classifiers in $\F$, such that on each of the $n$ distributions, at least one of the classifiers achieves a population error below $\eps$.

Our first result is a general algorithm that efficiently learns under the $(k, \eps)$-realizability assumption.

\begin{theorem}\label{thm:sample-upper-general}
    Suppose that $\D_1, \D_2, \ldots, \D_n$ are $(k, \eps)$-realizable with respect to hypothesis class $\F$. For any $\delta > 0$, there is an $(8\eps, \delta)$-PAC algorithm with sample complexity
    \[
        O\left(\frac{kd\log(n/k)\log(1/\eps)}{\eps} + \frac{n\log k\log(1/\eps) + n\log(n/\delta)}{\eps}\right).
    \]
\end{theorem}

Viewing $\eps$ and $\delta$ as constants, the sample complexity reduces to $O(kd\log(n/k) + n\log n)$. When $d$ is large, the overhead in the sample complexity is $k\log(n/k)$, which interpolates between the $O(\log n)$ overhead at $k = 1$ (shown by~\cite{BHPQ17}) and the $O(n)$ overhead at $k = n$ (where the $n$ learning tasks are essentially unrelated, and a linear overhead is unavoidable).

The factor $8$ in the PAC guarantee can be replaced by any fixed constant that is strictly greater than $1$, at the cost of a different hidden constant in the sample complexity. This follows from straightforward modifications to our proof.

While we prove Theorem~\ref{thm:sample-upper-general} (as well as the other positive results in the paper) under the assumption that the learning algorithm is given the value of $k$, this assumption can be removed via a standard doubling trick: We consider a sequence of guesses on the value of $k$: $1 = k_1 < k_2 < k_3 < \cdots$, where each $k_{i+1}$ is the smallest value of $k$ such that the sample complexity bound is at least twice the bound for $k_i$. Then, we run the learning algorithm with $k$ set to $k_1, k_2 - 1, k_2, k_3 - 1, k_3, \ldots$ in order, and terminate the algorithm as soon as we are convinced that the actual $k$ is larger than the current guess. This procedure succeeds as soon as the guess exceeds the actual value of $k$, and the sample complexity only increases by a constant factor.

We prove Theorem~\ref{thm:sample-upper-general} using the following natural augmentation of the instance space and the hypothesis class.

\begin{definition}[$(G, k)$-Augmentation]\label{def:augmented-class}
Let $\F$ be a hypothesis class over instance space $\X$. For finite set $G$ and $k \in [|G|]$, the $(G, k)$-augmentation of $\F$ is the hypothesis class $\F_{G,k}$ over $\X' \coloneqq G \times \X$ defined as:
\[
    \F_{G,k} \coloneqq \left\{g_{f, c}: f \in \F^k, c \in [k]^G\right\},
\]
where for $f = (f_1, f_2, \ldots, f_k)$ and $c = (c_i)_{i \in G}$, $g_{f,c}$ is the function that maps $(i, x) \in \X'$ to $f_{c_i}(x)$.

When $G = [n]$ for some integer $n$, we use the shorthands ``$\F_{n,k}$'' and ``$(n, k)$-augmentation''.
\end{definition}

Definition~\ref{def:augmented-class} becomes more natural in light of the following observation. For each $i \in [n]$, let $\D'_i$ be the distribution of $((i, x), y)$ when $(x, y)$ is drawn from $\D_i$. Then, for any $f \in \F^k$ and $c \in [k]^n$, we have
\[
    \err_{\D'_i}(g_{f,c})
=   \pr{((i, x), y) \sim \D'_i}{g_{f,c}(i, x) \ne y}
=   \pr{(x, y) \sim \D_i}{f_{c_i}(x) \ne y}
=   \err_{\D_i}(f_{c_i}).
\]
In particular, when distributions $\D_1, \D_2, \ldots, \D_n$ are $(k, \eps)$-realizable w.r.t.\ $\F$, by definition, there exist $k$ classifiers $f_1, f_2, \ldots, f_k \in \F$ and $n$ numbers $c_1, c_2, \ldots, c_n \in [k]$ such that $\err_{\D_i}(f_{c_i}) \le \eps$. Then, the corresponding $g_{f,c}$ has a population error of at most $\eps$ on every $\D'_i$. This reduces the problem to an instance of collaborative learning on hypothesis class $\F_{n,k}$ and distributions $\D'_1$ through $\D'_n$, with a \emph{single} unknown classifier that is simultaneously $\eps$-accurate for all $\D'_i$ (i.e., the $(1, \eps)$-realizability assumption).

Our proof of Theorem~\ref{thm:sample-upper-general} first upper bounds the VC dimension of $\F_{n,k}$ by a function of $n$, $k$, and the VC dimension of $\F$. Then, we adapt an algorithm of~\cite{BHPQ17} to achieve the sample complexity bound.

\paragraph{A sample complexity lower bound.} Complementary to Theorem~\ref{thm:sample-upper-general}, our next result shows that the sample complexity can be lower bounded in terms of the sample complexity for the $(1, 0)$-realizable case.

\begin{theorem}\label{thm:sample-lower-bound}
    Let $m(n, d, \eps, \delta)$ denote the optimal sample complexity of $(\eps,\delta)$-learning on a hypothesis class of VC dimension $d$ and $n$ distributions that are $(1, 0)$-realizable. Then, under the $(k, 0)$-realizable assumption, the sample complexity is lower bounded by $\Omega(k)\cdot m\left(\lfloor n / k\rfloor, d, \eps, O(\delta / k)\right)$.
\end{theorem}

Theorem~3.1 of~\cite{BHPQ17} bounds $m(n, d, \eps, \delta)$ by
\[
    O\left(\frac{\log n}{\eps}\left((d + n)\log(1/\eps) + n\log(n / \delta)\right)\right),
\]
which contains a $(d\log n) / \eps$ term. Assuming that this term is unavoidable (i.e., $m(n, d, \eps, \delta) = \Omega((d\log n) / \eps))$, by Theorem~\ref{thm:sample-lower-bound}, we have a lower bound of $\Omega\left(\frac{kd\log(n/k)}{\eps}\right)$ for the $(k, 0)$-realizable case. In other words, the leading term of the sample complexity in Theorem~\ref{thm:sample-upper-general} is necessary. Proving such an $\Omega(d\log n)$ lower bound, however, is still an open problem, even under the additional restriction that the learner must output the same function for all the $n$ distributions (see Problem~2 in the COLT'23 open problem of~\cite{AHZ23}).

\paragraph{Intractability of ERM and proper learning.} A downside of Theorem~\ref{thm:sample-upper-general} is that the learning algorithm might not be \emph{computationally} efficient, even if there is a computationally efficient learner for $\F$ in the usual PAC learning setup. Concretely, our learning algorithm requires Empirical Risk Minimization (ERM) on $\F_{n,k}$, the $(n, k)$-augmentation of $\F$. The straightforward approach involves enumerating all partitions of $[n]$ into $k$ sets, which takes exponential time.\footnote{There is a faster algorithm via dynamic programming, though its runtime is still $2^{\Omega(n)}$.}

Our next result shows that this ERM problem generalizes certain intractable discrete optimization problems, and is unlikely to be efficiently solvable. We first give a formal definition of the an ERM oracle.

\begin{definition}[ERM Oracle]\label{def:ERM-oracle}
    An ERM oracle for hypothesis class $\F \subseteq \{0, 1\}^{\X}$ is an oracle that, given any dataset $S$, returns $f^* \in \argmin_{f \in \F}\err_S(f)$.
\end{definition}

To state the hardness result rigorously, we need to consider a parametrized family of hypothesis classes instead of a fixed one.

\begin{definition}[Regular Hypothesis Family]\label{def:regular-hypothesis-family}
    A regular hypothesis family is $\{(\X_d, \F_d)\}_{d \in \N}$ that satisfies the following for every $d$:
    \begin{itemize}
        \item $\F_d$ is a collection of binary functions over $\X_d$ with VC dimension at least $d$.
        \item There is an efficient algorithm that, given $d$, outputs $x_1$, $x_2$, $\ldots$, $x_d \in \X_d$ that are shattered by $\F_d$.
    \end{itemize}
\end{definition}

\begin{remark}
    The first condition prevents the family from containing only simple classes with bounded VC dimensions. The second condition allows us to efficiently find witnesses for the VC dimension. Note that the second condition holds for natural hypothesis classes such as halfspaces and parity functions, the VC dimension of which can be lower bounded in a constructive way.
\end{remark}

We will show that the following decision version of ERM is already hard for $\F_{n,k}$: instead of finding $f^* \in \argmin_{f \in \F_{n,k}}\err_S(f)$, we are only required to decide whether $\min_{f \in \F_{n,k}}\err_S(f)$ is $0$ or not.

\begin{problem}[ERM over Augmented Classes]\label{prob:ERM}
    For a regular hypothesis family $\{(\X_d, \F_d)\}_{d \in \N}$, an instance of the ERM problem consists of parameters $(d, n, k)$ and $n$ datasets $S_1, S_2, \ldots, S_n \subseteq \X_d \times \{0, 1\}$. The goal is to decide whether there exist classifiers $f_1, f_2, \ldots, f_k \in \F_d$ such that for every $i \in [n]$, $\min_{j \in [k]}\err_{S_i}(f_j) = 0$.
\end{problem}

\begin{remark}
    Problem~\ref{prob:ERM} is equivalent to deciding whether there exists a classifier $f \in \F_{n,k}$ with a zero training error on the dataset $\{((i, x), y): i \in [n], (x, y) \in S_i\}$. Therefore, if we could efficiently implement the ERM oracle for $\F_{n,k}$, we would be able to solve Problem~\ref{prob:ERM} efficiently as well.
\end{remark}

Now we are ready to state our intractability result.

\begin{theorem}\label{thm:hardness-of-ERM}
    For any regular hypothesis family, ERM over augmented classes (Problem~\ref{prob:ERM}) is $\NP$-hard for any $k \ge 3$. Furthermore, there exists a regular hypothesis family on which Problem~\ref{prob:ERM} is polynomial-time solvable for $k = 1$ but $\NP$-hard for $k = 2$.
\end{theorem}

One might argue that Theorem~\ref{thm:hardness-of-ERM} only addresses the worst case, and does not exclude the possibility of efficiently implementing ERM (with high probability) over datasets that are randomly drawn. In Section~\ref{sec:distributional-ERM}, we state and prove a ``distributional'' analogue of Theorem~\ref{thm:hardness-of-ERM}, which shows that it is also unlikely for an efficient (and possibly randomized) algorithm to succeed on randomly drawn samples.

Recall that a proper learner is one that always returns hypotheses in the hypothesis class. In our setup, we say that a learning algorithm is \emph{strongly proper} if, when executed under the $(k, \eps)$-realizability assumption, it always outputs $n$ functions $\hat f_1, \ldots, \hat f_n \in \F$ such that $|\{\hat f_1, \ldots, \hat f_n\}| \le k$. Note that the $(k, \eps)$-realizability assumption implies that it is always possible to find accurate classifiers that satisfy this constraint. Unfortunately, our proof of Theorem~\ref{thm:hardness-of-ERM} also implies that, unless $\RP = \NP$, no strongly proper learner can be computationally efficient in general.

\paragraph{Efficient algorithms for special cases.} Despite the computational hardness in the general case, we identify two special cases in which computationally efficient learners exist, assuming an efficient ERM for $\F$.

The first case is when the $n$ data distributions share the same marginal over $\X$.

\begin{theorem}\label{thm:same-marginal}
    Suppose that $\D_1, \D_2, \ldots, \D_n$ are $(k, \eps)$-realizable and have the same marginal distribution on $\X$. Fix constant $\alpha > 0$. For any $\delta > 0$, there is a $((3 + \alpha)\eps, \delta)$-PAC algorithm that  runs in $\poly(n, k, 1 / \eps, \log(1 / \delta))$ time, makes at most $k$ calls to an ERM oracle for $\F$, and has a sample complexity of
    \[
        O\left(\frac{kd\log(1/\eps)}{\eps} + \frac{n\log(n/\delta)}{\eps}\right).
    \]
\end{theorem}

Our algorithm for the theorem above follows a similar approach to the lifelong learning algorithms of~\cite{BBV15,PU16}.

Our next positive result applies to hypothesis classes that are \emph{$2$-refutable} in the sense that whenever a dataset cannot be perfectly fit by $\F$, it contains two labeled examples that explain this inconsistency.

\begin{definition}[$2$-Refutability]
    A hypothesis class $\F \subseteq \{0, 1\}^{\X}$ is $2$-refutable if, for any dataset $S$ such that $\min_{f \in \F}\err_S(f) > 0$, there is $S' = \{(x_1, y_1), (x_2, y_2)\} \subseteq S$ such that $\min_{f \in \F}\err_{S'}(f) > 0$.
\end{definition}

The following gives examples of natural hypothesis classes that are $2$-refutable, and shows that $2$-refutability is preserved under certain operations.
\begin{example}
    The following hypothesis classes are $2$-refutable:
    \begin{itemize}
        \item $\F = \{0, 1\}^{\X}$. Any dataset that cannot be perfectly fit by $\F$ must contain both $(x, 0)$ and $(x, 1)$ for some $x \in \X$.
        \item $\F = \{f: \X \to \{0, 1\}: \sum_{x \in \X}f(x) \le 1\}$. Any dataset that cannot be perfectly fit by $\F$ must contain $(x_1, 1)$ and $(x_2, 1)$ for different $x_1, x_2 \in \X$.
        \item $\F = \{f' \circ g: f' \in \F'\}$, where $\F' \subseteq \{0, 1\}^{\X'}$ is $2$-refutable and $g: \X \to \X'$ is fixed.
        \item $\F = \{f' \oplus g: f' \in \F'\}$, where $\F' \subseteq \{0, 1\}^{\X}$ is $2$-refutable, $g:\X \to \{0, 1\}$ is fixed, and $\oplus$ denotes pointwise XOR.
    \end{itemize}
\end{example}

Assuming that the hypothesis class is $2$-refutable and the data distributions are $(k, 0)$-realizable, ERM on the augmented class $\F_{n,k}$ gets reduced to graph coloring, in light of the following definition and simple lemma.

\begin{definition}[Conflict Graph]\label{def:conflict-graph}
    The conflict graph induced by datasets $S_1, S_2, \ldots, S_n$ and hypothesis class $\F$ is an undirected graph $G = ([n], E)$, where $\{i, j\} \in E$ if and only if $\min_{f \in \F}\err_{S_i \cup S_j}(f) > 0$.
\end{definition}

\begin{lemma}\label{lemma:conflit-graph-IS}
    Let $\F$ be a $2$-refutable hypothesis class. Datasets $S_1, \ldots, S_n$ satisfy $\min_{f \in \F}\err_{S_i}(f) = 0$ for every $i \in [n]$. Let $V$ be an independent set in the conflict graph induced by $S_1, S_2, \ldots, S_n$ and $\F$. Then, for $S' = \bigcup_{i \in V}S_i$, it holds that $\min_{f \in \F}\err_{S'}(f) = 0$.
\end{lemma}

\begin{proof}
    Suppose for a contradiction that $\min_{f \in \F}\err_{S'}(f)$ is non-zero. Since $\F$ is $2$-refutable, there exist $i_1, i_2 \in V$, $(x_1, y_1) \in S_{i_1}$ and $(x_2, y_2) \in S_{i_2}$ such that no classifier in $\F$ correctly labels both examples. If $i_1 = i_2$, this contradicts the assumption $\min_{f \in \F}\err_{S_{i_1}}(f) = 0$. If $i_1 \ne i_2$, $i_1$ and $i_2$ must be neighbours in the conflict graph, which contradicts the independence of $V$.
\end{proof}

Assuming that the datasets are drawn from distributions are $(k, 0)$-realizable, the induced conflict graph must be $k$-colorable. If we could find a valid $k$-coloring efficiently, each color corresponds to an independent set of the graph. By Lemma~\ref{lemma:conflit-graph-IS}, we can call the ERM oracle for $\F$ to find a consistent function. Combining the functions for the $k$ different colors gives a solution to the ERM problem over the augmented class $\F_{n,k}$.

Unfortunately, graph coloring is $\NP$-hard when $k \ge 3$. Nevertheless, there are efficient algorithms for \emph{approximate coloring}, i.e., color a graph using a few colors, when the graph is promised to be $k$-colorable for some small $k$. The definition below together with Theorem~\ref{thm:refutable} gives a way of systematically translating an approximate coloring algorithm into an efficient algorithm for collaborative learning.

\begin{definition}\label{def:coloring-constants}
    For $k \ge 3$, let $c^*_k \in (0, 1]$ denote any constant such that any $k$-colorable graph with $n$ vertices can be efficiently colored with $O(n^{c^*_k})$ colors.
\end{definition}

A result of Karger, Motwani and Sudan~\cite{KMS98} shows that we can take $c^*_k = 1 - \frac{3}{k+1} + \eps$ for any $\eps > 0$. For $k = 3$, a more recent breakthrough of Kawarabayashi and Thorup~\cite{KT17} gives $c^*_3 = 0.19996$.

\begin{theorem}\label{thm:refutable}
    Suppose that $\D_1, \D_2, \ldots, \D_n$ are $(k, 0)$-realizable with respect to a $2$-refutable hypothesis class $\F$. For any $\delta > 0$, there is an $(\eps, \delta)$-PAC algorithm that  runs in $\poly(n, k, 1 / \eps, \log(1 / \delta))$ time, makes $\poly(n)$ calls to an ERM oracle for $\F$, and has a sample complexity of
    \[
        O\left(\frac{d\log(1/\eps) + n}{\eps}\cdot \log n + \frac{n\log(1/\delta)}{\eps}\right)
    \]
    if $k = 2$, and
    \[
        O\left(\frac{d\log(1/\eps) + n}{\eps}\cdot n^{c^*_k} + \frac{n\log(1/\delta)}{\eps}\right)
    \]
    if $k \ge 3$.
\end{theorem}

Note that when $k = 2$, the sample complexity is as good as the one in Theorem~\ref{thm:sample-upper-general}. When $k \ge 3$, the overhead increases from $\log n$ to $\poly(n)$. Nevertheless, this overhead is still sub-linear in $n$ for any fixed $k$. 

\subsection{Related Work}\label{sec:related-work}
\paragraph{Collaborative learning.} Most closely related to our work are the previous studies of Collaborative PAC Learning~\cite{BHPQ17,nguyen2018improved,chen2018tight,Qiao18} and related fields such as multi-task learning~\cite{hanneke2022no}, multi-distribution learning~\cite{haghtalab2022demand,AHZ23,Peng23,zhang2023optimal}, federated learning~\cite{mcmahan2017communication,mohri2019agnostic,cheng2023federated} and multi-source domain adaptation~\cite{mansour2008domain,konstantinov2019robust,mansour2021theory}.

In multi-task learning/multi-source domain adaptation, there are $n$ distributions, each with a fixed number of samples, and our goal is to use these samples to learn a hypothesis that has a small risk on some target distribution. A line of works~\cite{ben2010theory,konstantinov2019robust,mansour2021theory} studied the generalization risk in this scenario, but their bounds all depend on the discrepancy among $n$ distributions and contain non-vanishing residual constants. To avoid this residual constant in the bounds,
\cite{hanneke2022no} considered a Bernstein condition assumption on the hypothesis class and some transferrability assumptions between the $n$ source distributions and the target distribution. They studied the minimax rate of this learning scenario, and gave a nearly-optimal adaptation algorithm. Specially, some works studied the multi-task linear regression~\cite{yang2020analysis,huang2023optimal}. \cite{yang2020analysis} considered linear regression from multiple distributions, where all tasks share the same input feature covariate $\mathbf{X}\in \R^{m\times d}$, but with different labeling functions. They designed the Hard Parameter Sharing estimator and established an excess risk upper bound of $O(\frac{d\sigma^2}{m} + \mathrm{heterogeneity})$. Recently, \cite{huang2023optimal} proposed the $s$-{\em sparse heterogeneity} assumption among the labeling functions in multi-task linear regression, and designed an algorithm which achieves an $O(\frac{s\sigma^2}{m_i} + \frac{d\sigma^2}{\sum_{i=1}^n m_i})$ excess risk on the $i$-th task. Notice that when $s$ is smaller than $d$, this bound is strictly sharper than individual learning bound $O(\frac{d\sigma^2}{m_i} )$. The difference between our scenario and multi-task learning is that we allow the learner to draw an arbitrary number of samples from each distribution, instead of assuming that each distribution only has a fixed number of samples. 

Federated learning is another relevant key learning scenario, where $n$ players with their own underlying distributions, and fixed number of samples drawn from them, aim at learning model(s) that can have small risk on everyone's distribution. A line of works aimed at studying its statistical properties~\cite{JMLR:v24:21-0224,cheng2023federated,mohri2019agnostic}. \cite{JMLR:v24:21-0224} studied the minimax risk of federated learning in logistic regression setting, and showed that the minimax risk is controlled by the heterogeneity among $n$ distributions and their labeling functions. \cite{cheng2023federated} studied the risk bound of federated learning in the linear regression setting, and in an asymptotic fashion when the dimension of the model goes to infinity. Similar to multi-task learning, federated learning also assumes each player only has fixed number of samples, and the analysis does not give a PAC learning bound.

Multi-distribution learning was recently proposed by~\cite{haghtalab2022demand}, where $n$ players try to learn a single model $\hat f$, that can have an $\epsilon$ excess error on the worst case distribution among $n$ players, i.e., $\max_{i\in[n]} \err_{\cD_i}(\hat f) \leq \epsilon + \mathrm{OPT}$. They gave an algorithm with sample complexity $O(\frac{d\log n}{\epsilon^2} + \frac{nd\log(d/\eps)}{\epsilon})$, and proved a lower bound of $\tilde \Omega(\frac{d+k}{\epsilon^2})$. Note that 
this is a more {\em pessimistic} learning guarantee than ours, since the value $\mathrm{OPT}$ can be very large.  Very recently, two concurrent papers~\cite{Peng23,zhang2023optimal} gave algorithms that match this lower bound, resolving some of the open problems formulated in~\cite{AHZ23}. 

\paragraph{Mixture learning from batches.} Another recent line of work~\cite{KSSKO20,KSKO20,DJKS23,JSKDO23} studied learning mixtures of linear regressions from data batches. In these setups, there are $k$ unknown linear regression models. Each \emph{data batch} consists of labeled examples produced by one of the linear models chosen randomly. This line of work gave trade-offs between the number of batches and the batch size in order for the parameters of the $k$ linear models to be efficiently learnable.

In comparison, our model allows the learner to adaptively sample from the data distributions, whereas the batch sizes are fixed in the model of learning with batches. We also note that the results for learning with batches require assumptions on the marginal distribution, such as Gaussianity or certain hypercontractivity and condition number properties. Also, except the recent work of~\cite{JSKDO23}, they all required the marginal distribution of the instance to be the same across all batches.

\paragraph{Computational hardness of learning.} There is a huge body of work on the computational hardness of learning. Early work along this line showed that, under standard complexity-theoretic assumptions, it is hard to properly and agnostically learn halfspaces~\cite{FGKP06,GR09} and boolean disjunctions~\cite{KSS92,Feldman06}. More recent work have obtained a finer-grained understanding of this computational hardness. It is now known that many natural hypothesis classes are hard to learn even under additional assumptions, e.g., learning halfspaces under Massart noise~\cite{DKMR22}, agnostically learning halfspaces under Gaussian Marginals~\cite{DKR23}, and properly learning decision trees using membership queries~\cite{KST23a,KST23b}.

\paragraph{Approximate coloring.} Approximate coloring is the problem of finding a valid coloring of a given graph with as few colors as possible. This line of work was initiated by Wigderson~\cite{Wigderson83}, who gave an efficient algorithm that colors a $3$-colorable graph with $n$ vertices using $O(\sqrt{n})$ colors. This result was later improved by a series of work~\cite{BR90,Blum94,KMS98,BK97,ACC06,Chlamtac07,KT17}. The best known upper bound of $O(n^{0.19996})$ is due to Kawarabayashi and Thorup~\cite{KT17}. The analogous problem for $k$-colorable graphs (where $k \ge 4$) has also been studied.

\section{Discussion on Open Problems}
\paragraph{Tighter sample complexity bounds.} The most obvious open problem is to either improve the $kd\log(n/k)/\eps$ term in the sample complexity bound in Theorem~\ref{thm:sample-upper-general} or prove a matching lower bound. In light of Theorem~\ref{thm:sample-lower-bound}, it is sufficient to prove a lower bound of $\Omega((d\log n) / \eps)$ for the personalized setup of collaborative learning (i.e., the $(1, 0)$-realizable case). Conversely, any improvement on this term implies a better algorithm for the $(1, 0)$-realizable case.

\paragraph{A stronger hardness result.} The $\NP$-hardness is proved either for the ERM problem (in Theorem~\ref{thm:hardness-of-ERM}), or against learners that are strongly proper in the sense that they always return at most $k$ different classifiers (recall the discussion in Section~\ref{sec:our-results}). Our result does not rule out efficient learners that are neither ERM-based nor strongly proper.\footnote{In fact, the distributions that we constructed in the proof of Theorem~\ref{thm:sample-lower-bound} can be easily learned by an improper algorithm.} Can we prove the intractability of sample-efficient learning directly, at least for specific hypothesis classes?

\paragraph{Conflict graphs with bounded degrees.} Our Theorem~\ref{thm:refutable} gives a computationally efficient learner based on approximate coloring. It is also known that $k$-colorable graphs with the maximum degree bounded by $\Delta$ can be colored with a smaller number of colors (e.g., $\tilde O(\Delta^{1/3})$ colors when $k = 3$~\cite{KMS98}). It is interesting to identify natural assumptions on the data distributions that ensure this small-degree property in the conflict graph, and explore whether that leads to a lower sample complexity.

\paragraph{Efficient learner for concrete hypothesis classes.} Even when $\F$ is simply the class of all binary functions on an instance space of size $d$ and the data distributions are $(k, 0)$-realizable for $k = 3$, we do not have a computationally efficient learner that achieves the information-theoretic bound in Theorem~\ref{thm:sample-upper-general}. For this setup, since $\F$ is $2$-refutable, Theorem~\ref{thm:refutable} gives an algorithm with sample complexity of roughly $d\cdot n^{0.19996} / \eps$. Can we improve the overhead from $\poly(n)$ to $\polylog(n)$ via an efficient learner?

\section{Sample Complexity Upper Bound}\label{sec:upper}
In this section, we prove Theorem~\ref{thm:sample-upper-general}, which upper bounds the sample complexity of collaborative learning under $(k, \eps)$-realizability. The key step in the proof is the following upper bound on the VC dimension of $\F_{n,k}$.

\begin{lemma}\label{lemma:vc-dim-bound}
    For any $n \ge k \ge 1$ and hypothesis class $\F$ of VC dimension $d$, the VC dimension of $\F_{n,k}$ is at most $O(kd + n\log k)$.
\end{lemma}

Lemma~\ref{lemma:vc-dim-bound} implies that in general, $\F_{G, k}$ has a VC dimension of $O(kd + |G|\log k)$. To gain some intuition behind the bound in Lemma~\ref{lemma:vc-dim-bound}, suppose that $\F$ is a finite class of size $2^d$. By Definition~\ref{def:augmented-class}, the size of $\F_{n,k}$ is at most $|\F|^k \cdot k^n = 2^{kd + n\log_2 k}$, and the bound in the lemma immediately follows. The actual proof, of course, needs to deal with the case that $\F$ is larger or even infinite.

The lemma improves a previous result of~\cite[Theorem 1]{CM12}, which upper bounds the VC dimension of $\F_{n,k}$ (called the class of ``hard $k$-shared task classifiers'') by $O(kd\log(nkd) + n\log k)$. Note that this bound can be looser than the one in Lemma~\ref{lemma:vc-dim-bound} by a logarithmic factor.

\begin{proof}[Proof of Lemma~\ref{lemma:vc-dim-bound}]
    We will show that, for some integer $m \ge kd$ to be chosen later, no $m$ instances in $\X' = [n] \times \X$ can be shattered by $\F_{n,k}$. This upper bounds the VC dimension of $\F_{n,k}$ by $m - 1$.

    Fix a set $S$ of $m$ elements in $\X'$. To bound the number of ways in which $S$ can be labeled by $\F_{n,k}$, we also fix $c^* \in [k]^n$ and focus on the classifiers in $\F_{n,k}$ associated with $c^*$. Note that $c^*$ naturally partitions $S$ into $S_1 \cup S_2 \cup \cdots \cup S_k$, where $S_j \coloneqq \{(i, x) \in S: c^*_i = j\}$. Furthermore, let $X_j \coloneqq \{x \in \X: (i, x) \in S_j,~\exists i \in [n]\}$ be the projection of $S_j$ to $\X$. Note that we have
    \[
        \sum_{j=1}^{k}|X_j|
    \le \sum_{j=1}^{k}|S_j|
    = |S| = m.
    \]
    
    Let $\Phi(\cdot)$ be the growth function of hypothesis class $\F$. Then, for each fixed $c^* \in [k]^n$, the number of ways in which $S$ can be labeled by classifiers in $\{g_{f,c^*}: f \in \F^k\} \subseteq \F_{n,k}$ is at most
    \[
        N_{c^*} \le \prod_{j=1}^k \Phi(|X_j|).
    \]
    By the Sauer-Shelah-Perles lemma, we have the following upper bound:
    \[
        \Phi(m) \le \overline{\Phi}(m) \coloneqq  \begin{cases}
            e^m, & m \le d,\\
            \left(\frac{em}{d}\right)^d, & m > d.
        \end{cases}
    \]
    It can be verified that the function $m \mapsto \ln\overline{\Phi}(m)$ is monotone increasing and concave on $[0, +\infty)$. It then follows from $\sum_{j=1}^{k}|X_j| \le m$ that
    \begin{align*}
        \ln N_{c^*}
    &\le k \cdot \frac{1}{k}\sum_{j=1}^{k}\ln\overline{\Phi}(|X_j|)\\
    &\le k \cdot \ln\overline{\Phi}\left(\frac{1}{k}\sum_{j=1}^{k}|X_j|\right) \tag{concavity}\\
    &\le k \cdot \ln\overline{\Phi}\left(\frac{m}{k}\right) \tag{monotonicity}\\
    &=  kd\ln\frac{em}{kd}. \tag{$m \ge kd$}
    \end{align*}
    Then, summing over the $k^n$ different choices of $c^*$, the logarithm of the growth function of $\F_{n,k}$ at $m$ is at most
    \[
        \ln\left(\sum_{c \in [k]^n}N_c\right)
    \le n\ln k + kd\ln\frac{em}{kd}.
    \]
    For some sufficiently large $m = O(n\log k + kd)$, the above is strictly smaller than $\ln (2^m)$, which means that the $m$ points in $S$ cannot be shattered by $\F_{n,k}$.
\end{proof}

Given Lemma~\ref{lemma:vc-dim-bound}, Theorem~\ref{thm:sample-upper-general} essentially follows from the learning algorithm of~\cite{BHPQ17} for the personalized setup, with slight modifications. For completeness, we state the algorithm and prove its correctness in Appendix~\ref{sec:upper-proofs}.

\section{Sample Complexity Lower Bound}
Our proof of Theorem~\ref{thm:sample-lower-bound} is based on a simple observation: learning $n$ distributions under $(k, 0)$-realizability is at least as hard as learning $k$ unrelated instances, where each instance consists of learning $n / k$ distributions under $(1, 0)$-realizability.

Our actual proof is essentially the same as the lower bound proof of~\cite{BHPQ17}, and formalizes the intuition above. Formally, assuming that a learning algorithm $\cA$ $(\epsilon,\delta)$-PAC learns $n$ distributions under $(k, 0)$-realizability, we use $\cA$ to construct another learner $\cA'$, which is $(\epsilon, O(\delta/k))$-PAC for $n / k$ distributions that are $(1, 0)$-realizable. Furthermore, the sample complexity of $\cA'$ is an $O(1/k)$ fraction of that of $\cA$. For completeness, we state the reduction in the following.

\begin{proof}[Proof of Theorem~\ref{thm:sample-lower-bound}]
    Fix parameters $n$, $k$, $d$, $\eps$, and $\delta$. Let $n' \coloneqq \lfloor n / k \rfloor$ and $\delta' \coloneqq \frac{10\delta}{9k}$. Let $\F$ be a hypothesis class with VC dimension $d$, and $\Dhard$ be a distribution over hard instances for collaborative learning on $n'$ distributions. Formally, $\Dhard$ is a distribution over $n'$ data distributions $(\D_1, \D_2, \ldots, \D_{n'})$ such that:
    \begin{itemize}
        \item Every $(\D_1, \ldots, \D_{n'})$ in the support of $\Dhard$ is $(1, 0)$-realizable with respect to $\F$.
        \item If a learning algorithm achieves an $(\eps, \delta')$-PAC guarantee when learning $\F$ on $\D_1, \ldots, \D'_{n'}$ drawn from $\Dhard$, it must take $m(n', d, \eps, \delta')$ samples in expectation.
    \end{itemize}

    Now, let $\A$ be an $(\eps, \delta)$-PAC learning algorithm over $k\cdot n'$ distributions under $(k, 0)$-realizability. For brevity, we relabel the $k \cdot n'$ distributions as $\D_{i,j}$ where $i \in [k]$ and $j \in [n']$. In the following, we construct another learning algorithm $\A'$ that learns $n'$ distributions (denoted by $\Dact_1, \ldots, \Dact_{n'}$) drawn from $\Dhard$ by simulating $\A$:

    \begin{enumerate}
        \item For each $i \in [k]$, independently draw $(\D_{i,1}, \ldots, \D_{i,n'})$ from $\Dhard$.
        \item Sample $i^*$ from $[k]$ uniformly at random.
        \item We simulate algorithm $\cA$ on distributions $(\D_{i,j})_{i \in [k], j \in [n']}$, except that $\D_{i^*,1}$ through $\D_{i^*,n'}$ are replaced by the $n'$ actual distributions $\Dact_1, \ldots, \Dact_{n'}$. In other words, whenever $\A$ requires a sample from $\D_{i,j}$, we truly sample from $\D_{i,j}$ if $i \ne i^*$; otherwise, we sample from $\Dact_j$, and forward the sample to $\A$.
        \item When $\A$ terminates and outputs $(\hat f_{i,j})_{i \in [k], j \in [n']}$, we test if $\err_{\D_{i,j}}(\hat f_{i,j}) \leq \epsilon$ holds for all $i \neq i^*$ and $j \in [n']$. If so, we output $\hat f_{i^*,1}$ through $\hat f_{i^*, n'}$ as the answer; otherwise, we repeat the procedure above.
    \end{enumerate}

    In each repetition of the above procedure, from the perspective of algorithm $\A$, it runs on $k \cdot n'$ distributions divided into $k$ groups, where each group consists of $n'$ distributions drawn from $\Dhard$. Clearly, the $k \cdot n'$ distributions together satisfy $(k, 0)$-realizability. Intuitively, it is impossible for $\A$ to tell the index $i^*$ that corresponds to the actual instance, so the actual instance only suffers from an $O(1/k)$ fraction of the error probability as well as the sample complexity.

    Let $M$ denote the expected number of samples that $\A$ draws on such a random instance. Analogous to Claims 4.3~and~4.4 from~\cite{BHPQ17}\footnote{While the two claims in~\cite{BHPQ17} were proved for a concrete construction of hard instances, the proof only relies on the symmetry and independence among the instances, and thus can be applied to our case without modification.}, we have the following guarantees of the constructed learner $\A'$:
    \begin{claim}
        Assuming $\delta \le 0.1$, on a random instance drawn from $\Dhard$, $\cA'$ achieves an $(\epsilon, \frac{10\delta}{9k})$-PAC guarantee and draws at most $\frac{10M}{9k}$ samples in expectation.
    \end{claim}
    By our assumption on $\Dhard$, we have $\frac{10M}{9k} \ge m(n', d, \epsilon, \delta')$. Hence, we conclude that $\cA$ takes at least
    \[\frac{9k}{10}\cdot m(n', d, \eps, \delta') = \Omega(k)\cdot  m(\lfloor n/k\rfloor,d,\epsilon, O(\delta/k))\]
    samples in expectation.
\end{proof}

\section{Evidence of Intractability}\label{sec:intractability}
In this section, we prove Theorem~\ref{thm:hardness-of-ERM} as well as a distributional version of it.

\subsection{Reduction from Graph Coloring}
We prove the first part (the $k \ge 3$ case) by a reduction from graph coloring, which is well-known to be $\NP$-hard.

\begin{proof}[Proof of Theorem~\ref{thm:hardness-of-ERM} (the first part)]
    Fix an arbitrary regular hypothesis family $\{(\X_d, \F_d)\}_{d \in \N}$. We will show that if, for any fixed $k \ge 3$, there is a polynomial-time algorithm that solves Problem~\ref{prob:ERM}, the same algorithm can be used to solve graph $k$-coloring efficiently. This implies the first part of the theorem.
    
    Given an instance $G = (V, E)$ of the $k$-coloring problem, we construct an instance of Problem~\ref{prob:ERM} with parameters $k$ and $d = n = |V|$. Without loss of generality, we assume $V = [n]$, since we can always relabel the vertices. By Definition~\ref{def:regular-hypothesis-family}, the VC dimension of $\F_d$ is at least $d = n$, and we can efficiently find $n$ instances $x_1, x_2, \ldots, x_n \in \X_d$ that are shattered by $\F_d$.

    For each $v \in [n]$, we define the $v$-th dataset as
    \[
        S_v \coloneqq \{(x_v, 1)\} \cup \{(x_u, 0): \{u, v\} \in E\}.
    \]
    We will show in the following that $G$ has a $k$-coloring if and only if the ERM instance is feasible, i.e., the $n$ datasets can be perfectly fit by $k$ classifiers from $\F_d$.
    
    \paragraph{From coloring to classifiers.} Suppose that $c: V \to [k]$ is a valid $k$-coloring of $G$. Since $\F_d$ shatters $x_1, x_2, \ldots, x_n$, we can find $f_1, f_2, \ldots, f_k \in \F_d$ such that $f_i(x_v) \coloneqq \1{c(v) = i}$ holds for every $i \in [k]$ and $v \in [n]$. Then, for every $v \in [n]$, the dataset $S_v$ is perfectly fit by classifier $f_{c(v)}$, since $f_{c(v)}(x_v) = \1{c(v) = c(v)} = 1$ and $f_{c(v)}(x_u) = \1{c(u) = c(v)} = 0$ for every neighbor $u$ of $v$.

    \paragraph{From classifiers to coloring.} Conversely, let $f_1, f_2, \ldots, f_k \in \F_d$ be $k$ classifiers such that each dataset $S_v$ is consistent with one of the classifiers. We can then choose a labeling $c: V \to [k]$ such that $S_v$ is perfectly fit by $f_{c(v)}$. Now we show that $c$ is a valid $k$-coloring. Indeed, suppose that $\{u, v\} \in E$ is an edge and $c(u) = c(v) = i$. Then, $f_i$ must correctly label both $(x_u, 1) \in S_u$ and $(x_u, 0) \in S_v$, which is impossible.

    Finally, note that the above reduction works for any $k \ge 3$ and any family of hypothesis classes, while $k$-coloring is $\NP$-hard for any $k \ge 3$. This proves the first part of Theorem~\ref{thm:hardness-of-ERM}.
\end{proof}

\subsection{Hardness under Two Labeling Functions}
Next, we deal with the $k = 2$ case. Since $2$-coloring can be efficiently solved, we must reduce from a different $\NP$-hard problem. Intuitively, we want the problem to correspond to partitioning a set into $k = 2$ parts. This motivates our reduction from (a variant of) subset sum.

\begin{proof}[Proof of Theorem~\ref{thm:hardness-of-ERM}, the second part]
    We start by defining the regular hypothesis family. For each integer $d \ge 1$, we consider the instance space $\X_d \coloneqq [d] \times \{0, 1, \ldots, 2^d\}$ and the following hypothesis class:
    \[
        \F_d \coloneqq \left\{f_{\theta}: \theta \in \{0, 1, \ldots, 2^d\}^d, \sum_{i=1}^{d}\theta_i \le 2^d\right\},
    \]
    where $f_\theta$ is defined as
    \[
        f_\theta(i, j) = \1{j \le \theta_i}.
    \]
    In other words, each $f_\theta \in \F_d$ can be viewed as a direct product of $d$ threshold functions, subject to that the thresholds sum up to at most $2^d$. It can be easily verified that $\{(\X_d, \F_d)\}_{d \in \N}$ satisfies Definition~\ref{def:regular-hypothesis-family}, since for every $d$, the instances $(1, 1), (2, 1), \ldots, (d, 1)$ are shattered by $\F_d$, witnessed by $\{f_{\theta}: \theta \in \{0, 1\}^d\} \subseteq \F_d$.

    \paragraph{Vanilla ERM is easy.} We first show that the $k = 1$ case of Problem~\ref{prob:ERM} is easy. Indeed, when $k = 1$, Problem~\ref{prob:ERM} reduces to deciding whether $S_1 \cup S_2 \cup \cdots \cup S_n \subseteq \X_d \times \{0, 1\}$ can be perfectly fit by a hypothesis in $\F_d$. By construction of $\F_d$, this can be easily done via the following two steps:
    \begin{itemize}
        \item First, check whether there exist $i \in [d]$ and $0 \le j_1 < j_2 \le 2^d$ such that both $((i, j_1), 0)$ and $((i, j_2), 1)$ are in the dataset. Also check whether the dataset contains $((i, 0), 0)$ for any $i \in [d]$. If either condition holds, report ``no solution''.
        \item Then, for each $i \in [d]$, let $\theta_i$ denote the largest value $j \in \{0, 1, \ldots, 2^d\}$ such that the dataset contains $((i, j), 1)$; let $\theta_i = 0$ if no such labeled example exists. If $\sum_{i=1}^{d}\theta_i \le 2^d$, the function $f_{\theta}$ is a valid solution; otherwise, report ``no solution''.
    \end{itemize}
    The correctness of this procedure is immediate given the definition of $\F_d$.
    \paragraph{Construction of datasets.} We consider a specific choice of the datasets: for each $i \in [n]$, the $i$-th dataset contains exactly one data point of form $(i, a_i)$ with label $1$. In the rest of the proof, the key observation is that the datasets with indices in $T \subseteq [n]$ can be simultaneously satisfied by a hypothesis in $\F_d$ if and only if $\sum_{i \in T}a_i \le 2^d$. The ERM problem for $k = 2$ is then equivalent to deciding whether $\{a_i: i \in [n]\}$ can be partitioned into two sets, each of which sums up to $\le 2^d$. This problem can be easily shown to be $\NP$-hard via a standard reduction from the subset sum problem.

    \paragraph{From subset sum to a special case.} We first reduce the general subset sum problem to a special case, in which the $n$ numbers sum up to $2^{n+1}$ and the target value is exactly $2^n$. Let $(\{a_i\}_{i \in [m]}, t)$ be an instance of subset sum (i.e., deciding whether there exists $T \subseteq [m]$ such that $\sum_{i \in T}a_i = t$). Let $s = \sum_{i=1}^{m}a_i$ and pick $n = \max\{m + 2, \lfloor \log_2 s\rfloor + 1\}$ such that $n - m \ge 2$ and $2^{n} > s$. We pad $n - m$ numbers to the instance, such that $a_{m+1} = a_{m+2} = \cdots = a_{n-2} = 0$, $a_{n-1} = 2^n - t$, $a_n = 2^n - (s - t)$. Note that the numbers now sum up to
    \[
        \sum_{i=1}^{n}a_i = s + (2^{n} - t) + [2^{n} - (s - t)] = 2^{n+1}.
    \]
    Also note that the size of the instance increases at most polynomially after the padding: A natural representation of the original subset sum instance $(\{a_i\}_{i \in [m]}, t)$ takes at least $m + \log_2 s$ bits. In the new instance, there are $n = O(m + \log s)$ numbers that sum up to $2^{n+1}$, so its representation takes at most $O(n^2)$ bits, which is at most quadratic in the size of the original instance.

    We claim that $(\{a_i\}_{i \in [n]}, 2^{n})$ has the same answer as $(\{a_i\}_{i \in [m]}, t)$. Indeed, if $\sum_{i\in T}a_i = t$ for some $T \subseteq [m]$, $T \cup\{n-1\}$ would be a feasible solution to the new instance. Conversely, suppose that $T \subseteq [n]$ is a subset such that $\sum_{i \in T}a_i = 2^{n}$. Since $a_{n-1} + a_{n} = 2^{n+1} - s > 2^{n}$, $T$ must contain exactly one of $n - 1$ and $n$. Without loss of generality, we have $n - 1 \in T$, and $T \setminus \{n - 1\}$ would then give $\sum_{i \in T \setminus \{n-1\}}a_i = 2^{n} - (2^{n} - t) = t$.

    \paragraph{From the special case to ERM.} Then, we construct a collaborative learning instance with $n$ distributions and set $d = n$ in the definition of $\X_d$ and $\F_d$. The $i$-th dataset only contains $((i, a_i), 1)$. We claim that the datasets can be fit by $k = 2$ functions in $\F_d$ if and only if the subset sum instance $(\{a_i\}_{i\in[n]}, 2^n)$ is feasible. For the ``if'' direction, suppose that $T \subseteq [n]$ satisfies $\sum_{i \in T}a_i = 2^n$. Then, we define $\theta^{(1)}$ and $\theta^{(2)}$ as:
    \[
        \theta^{(1)}_i = \begin{cases}
            a_i, & i \in T,\\
            0, & i \notin T,
        \end{cases}
        \quad \text{and} \quad
        \theta^{(2)}_i = \begin{cases}
            a_i, & i \notin T,\\
            0, & i \in T.
        \end{cases}
    \]
    Clearly, both $f_{\theta^{(1)}}$ and $f_{\theta^{(2)}}$ are in $\F_d$. The $i$-th dataset is consistent with the first function if $i \in T$ and the second otherwise.

    Conversely, suppose the datasets can be perfectly fit by two hypotheses $f_{\theta^{(1)}}, f_{\theta^{(2)}} \in \F_d$. Let $T \coloneqq \{i \in [n]: f_{\theta^{(1)}}(i, a_i) = 1\}$ be the indices of the data that are consistent with the former. We then have $\theta^{(1)}_i \ge a_i$ for every $i \in T$ and thus $\sum_{i \in T}a_i \le \sum_{i \in T}\theta^{(1)}_i \le 2^n$. The same argument, when applied to $[n] \setminus T$ and $\theta^{(2)}$, implies $\sum_{i \in [n]\setminus T}a_i \le \sum_{i \in [n]\setminus T}\theta^{(2)}_i \le 2^n$. Since the $a_i$'s sum up to $2^{n+1}$, we conclude that each summation must be equal to $2^n$, i.e., $T$ is a feasible solution to the subset sum instance.
\end{proof}

\subsection{Hardness of the Distributional Version of ERM}\label{sec:distributional-ERM}
As we mentioned earlier, Theorem~\ref{thm:hardness-of-ERM} and its proof are arguably of a worst-case nature, and does not exclude the possibility of efficiently performing ERM (with high probability) over datasets that are randomly drawn. Indeed, for the first part of the proof (via a reduction from graph coloring), when the graph $G = (V, E)$ is dense, each dataset constructed in the proof would be of size $\Omega(|V|) = \Omega(d)$, whereas in the context of sample-efficient collaborative learning, the datasets tend to be much smaller.

Unfortunately, we can also prove the hardness of this distributional version, which we formally define below.

\begin{problem}[ERM over Randomly Drawn Datasets]\label{prob:ERM-distr}
    This is a variant of Problem~\ref{prob:ERM}, in which we specify $n$ distributions $\D_1, \D_2, \ldots, \D_n$ over $\X_d \times \{0, 1\}$ and a parameter $m$. Each dataset $S_i$ consists of $m$ samples independently drawn from $\D_i$.
\end{problem}

We first note that the $k = 2$ part of Theorem~\ref{thm:hardness-of-ERM} still holds for Problem~\ref{prob:ERM-distr}. This is because our proof shows that Problem~\ref{prob:ERM} is hard (for a specific hypothesis family) even if all the datasets are of size $1$. Then, the same construction can be used to reduce subset sum to Problem~\ref{prob:ERM-distr}, in which each distribution is a degenerate distribution.

To prove the hardness of Problem~\ref{prob:ERM-distr} when $k \ge 3$ and the hypothesis family is arbitrary, we will reduce from the following variant of the graph $k$-coloring problem, in which the graph is promised to have degrees bounded by $O(k)$.

\begin{problem}\label{prob:sparse-coloring}
    Given a graph $G = (V, E)$ with maximum degree at most $2k - 1$, decide whether $G$ can be $k$-colored.
\end{problem}

\begin{lemma}\label{lemma:hardness-coloring-sparse}
    Problem~\ref{prob:sparse-coloring} is $\NP$-hard.
\end{lemma}
\begin{proof}
    We reduce from the usual $k$-coloring. Let $G = (V, E)$ be an instance of $k$-coloring. We will give an efficient algorithm that transforms $G$ into a graph $G'$ that is a valid instance for Problem~\ref{prob:sparse-coloring}. Furthermore, $G$ is $k$-colorable if and only if $G'$ is $k$-colorable. This immediately implies the $\NP$-hardness of Problem~\ref{prob:sparse-coloring}.

    For each node $v \in V$, we split it into $|V| - 1$ copies $v^{(1)}, v^{(2)}, \ldots, v^{(|V| - 1)}$. We also add $|V| - 2$ cliques of size $k - 1$, denoted by $C_{v,1}, C_{v,2}, \ldots, C_{v, |V|-2}$. Every vertex in clique $C_{v, i}$ is also linked to $v^{(i)}$ and $v^{(i+1)}$. Note that this forces $v^{(1)}, v^{(2)}, \ldots, v^{(|V| - 1)}$ to take the same color in a valid $k$-coloring. Then, for every edge $\{u, v\} \in E$, we link some copy $u^{(i)}$ of $u$ to another copy $v^{(j)}$ of $v$, so that no copy of any vertex is used twice. The resulting graph $G' = (V', E')$ satisfies
    \[
        |V'| = |V|\cdot\left[|V| - 1 + (|V| - 2)\cdot (k - 1)\right] = O(k|V|^2),
    \]
    and the maximum degree is $1 + 2(k-1) = 2k - 1$. The equivalence between the $k$-colorability of $G$ and $G'$ is immediate from our construction.
\end{proof}

Now we prove a strengthening of Theorem~\ref{thm:hardness-of-ERM}.
\begin{theorem}\label{thm:hardness-of-ERM-distr}
    Unless $\NP = \RP$, no polynomial-time (possibly randomized) algorithm for Problem~\ref{prob:ERM-distr} achieves the following guarantee for $k \ge 3$ and $m = \Omega(k\log n)$: With probability at least $1 / \poly(d, n, m)$ over the randomness in $S_1, \ldots, S_n$, if $S_1, \ldots, S_n$ admits a feasible solution, the algorithm outputs a feasible solution with probability at least $1/\poly(d, n, m)$.
\end{theorem}

Indeed, the guarantee required in Theorem~\ref{thm:hardness-of-ERM-distr} seems minimal for a useful ERM oracle: It only needs to succeed on a non-negligible fraction of instances, and the definition of ``success'' is merely to be able to output a feasible solution (if one exists) with a non-negligible probability. Still, it is unlikely to achieve such a guarantee efficiently under the standard computational hardness assumption of $\NP \ne \RP$.

\begin{proof}
    We prove the contrapositive: the existence of such algorithms implies $\NP = \RP$. Suppose that $\A$ is an efficient algorithm with the desired guarantees. We derive an efficient algorithm for Problem~\ref{prob:sparse-coloring}.
    
    Given an instance of Problem~\ref{prob:sparse-coloring}, the reduction from the proof of Theorem~\ref{thm:hardness-of-ERM} produces $n = |V|$ datasets $\hat S_1, \hat S_2, \ldots, \hat S_n$, each of size at most $2k$. We define the $i$-th data distribution as the uniform distribution over $\hat S_i$. When $m$ samples are drawn from each $\D_i$, every element in the support of every $\D_i$ appears at least once, except with probability at most
    \[
        n \cdot 2k \cdot \left(1 - \frac{1}{2k}\right)^m
    \le 2kn\cdot\exp\left(-\frac{m}{2k}\right),
    \]
    which can be made much smaller than the $1/\poly(d, n, m)$ term in the theorem statement for some appropriate $m = \Omega(k \log n)$. In other words, except with a negligibly small probability, the randomly drawn datasets $S_1, \ldots, S_n$ coincide with the intended datasets $\hat S_1, \ldots, \hat S_n$ (when both are viewed as sets rather than multisets).
    
    Then, the hypothetical algorithm $\A$ for Problem~\ref{prob:ERM-distr} must output the correct answer with probability larger than $1/\poly(d, n, m)$, when the sampled datasets are $\hat S_1, \ldots, \hat S_n$. We repeat $\A$ on $\hat S_1, \ldots, \hat S_n$ for $O(\poly(d, n, m))$ times and check whether it ever outputs a feasible solution. If so, we output ``Yes''; we output ``No'' otherwise. This gives an efficient randomized algorithm for Problem~\ref{prob:sparse-coloring} that: (1) when the input graph is $k$-colorable, outputs ``Yes'' with probability $\ge 1/2$; (2) when the graph is not $k$-colorable, always outputs ``No''. This implies $\NP = \RP$ in light of Lemma~\ref{lemma:hardness-coloring-sparse}.
\end{proof}

\section{An Efficient Algorithm for Identical Marginals}
In this section, we prove Theorem~\ref{thm:same-marginal}, which addresses the special case that $\D_1, \D_2, \ldots, \D_n$ share the same marginal distribution over $\X$. In this case, we show that there is a simple algorithm that efficiently clusters the distributions and learn accurate classifiers using $\ll nd$ samples. The algorithm is formally defined as Algorithm~\ref{algo:same-marginal}, and follows a similar approach to the lifelong learning algorithms of~\cite{BBV15,PU16}.

\begin{algorithm}[ht]
    \caption{Efficient Clustering under Identical Marginals}\label{algo:same-marginal}
    \begin{algorithmic}[1]
    \STATE \textbf{Input:} Hypothesis class $\F$. Sample access to $\D_1, \ldots, \D_n$. Parameters $k, \eps, \delta, \alpha, c$.
    \STATE \textbf{Output:} Hypotheses $\hat f_1, \hat f_2, \ldots, \hat f_n$.
    \STATE $F \gets \emptyset$;
    \FOR{$i \in [n]$}
        \STATE Draw $c \cdot \frac{\ln(n|F|/\delta)}{\eps}$ samples from $\D_i$ to form dataset $S$; \alglinelabel{line:same-marginal-test}
        \STATE $\hat f \gets \argmin_{f \in F}\err_S(f)$;
        \IF{$\err_S(\hat f) \le (3 + \frac{2}{3}\alpha)\eps$}
            \STATE $\hat f_i \gets \hat f$;
        \ELSE
            \STATE Draw $c \cdot \frac{d\ln(1/\eps) + \ln[(|F| + 1)/\delta]}{\eps}$ samples from $\D_i$ to form dataset $S$;
            \STATE $\hat f_i \gets \argmin_{f \in \F}\err_S(f)$; \alglinelabel{line:same-marginal-naive}
            \STATE $F \gets F \cup \{\hat f_i\}$;
        \ENDIF
    \ENDFOR
    \STATE \textbf{Return:} $\hat f_1, \hat f_2, \ldots, \hat f_n$;
    \end{algorithmic}
\end{algorithm}

The algorithm maintains a list $F$ of classifiers from class $\F$. For each distribution $\D_i$, we first test whether any classifier in $F$ is accurate enough on it. If so, we set $\hat f_i$ as the best classifier in $F$; otherwise, we draw fresh samples to learn an accurate classifier for $\D_i$ and add it to $\F$.

Now we analyze Algorithm~\ref{algo:same-marginal}. We first define a good event that implies the accuracy and sample efficiency of the algorithm.

\begin{definition}
    Let $\goodevent$ denote the event that the following happen simultaneously when Algorithm~\ref{algo:same-marginal} is executed:
    \begin{itemize}
        \item Whenever Line~\ref{line:same-marginal-test} is reached, it holds for every $f \in F$ that: (1) $\err_{\D_i}(f) \le (3 + \alpha/3)\eps$ implies $\err_{S}(f) \le \left(3 + \frac{2}{3}\alpha\right)\eps$; (2) $\err_{S}(f) > (3 + \alpha)\eps$ implies $\err_{\D_i}(f) > \left(3 + \frac{2}{3}\alpha\right)\eps$.
        \item Whenever Line~\ref{line:same-marginal-naive} is reached, it holds that $\err_{\D_i}(\hat f_i) \le (1 + \alpha / 3)\eps$.
    \end{itemize}
\end{definition}

We show that $\goodevent$ happens with high probability, and implies that Algorithm~\ref{algo:same-marginal} is $((3 + \alpha)\eps, \delta)$-PAC and has a small sample complexity.

\begin{lemma}\label{lemma:same-marginal-good-event}
    For any fixed $\alpha > 0$, there exists a sufficiently large $c > 0$ such that when Algorithm~\ref{algo:same-marginal} is executed with parameters $\alpha$ and $c$, $\pr{}{\goodevent} \ge 1 - \delta$.
\end{lemma}

\begin{proof}
    We upper bound the probability for each of the two conditions in $\goodevent$ to be violated.

    For the first condition, we fix $i \in [n]$ and $f \in F$. We note that $\err_S(f)$ is the average of $c\cdot\frac{\ln(n|F|/\delta)}{\eps}$ independent Bernoulli random variables, each with expectation $\err_{\D_i}(f)$. By a Chernoff bound, for sufficiently large constant $c$ (that depends on $\alpha$), it holds with probability $1 - \frac{\delta}{2n|F|}$ that: (1) $\err_{\D_i}(f) \le (3 + \alpha/3)\eps$ implies $\err_{S}(f) \le \left(3 + \frac{2}{3}\alpha\right)\eps$; (2) $\err_{\D_i}(f) > (3 + \alpha)\eps$ implies $\err_{S}(f) > \left(3 + \frac{2}{3}\alpha\right)\eps$. By a union bound over all $f \in F$, the first condition of $\goodevent$ holds for a specific $i \in [n]$ with probability at least $1 - \delta / (2n)$. By another union bound, the first condition holds for all $i \in [n]$ with probability at least $1 - \delta / 2$.

    For the second condition, suppose that we reach Line~\ref{line:same-marginal-naive} at the $i$-th iteration of the for loop, and $|F| = r - 1$. Recall that the $(k, \eps)$-realizability of $\D_1$ through $\D_n$ implies that there exists $f \in \F$ such that $\err_{\D_i}(f) \le \eps$. Then, by Theorem~5.7 of~\cite{AB99}, for some sufficiently large $c$, we have $\err_{\D_i}(\hat f_i) \le (1 + \alpha / 3)\eps$ with probability at least $1 - \delta / (4r^2)$. Since $|F|$ is incremented whenever Line~\ref{line:same-marginal-naive} is reached, we only need a union bound over all $r = 1, 2, \ldots, n$, and the probability for the second condition to be violated is upper bounded by
    \[
        \sum_{r=1}^{n}\frac{\delta}{4r^2}
    \le \frac{\delta}{4}\sum_{r=1}^{+\infty}\frac{1}{r^2}
    =   \frac{\delta}{4}\cdot\frac{\pi^2}{6}
    <   \frac{\delta}{2}.
    \]

    Finally, yet another union bound gives $\pr{}{\goodevent} \ge 1 - \delta / 2 - \delta / 2 = 1 - \delta$.
\end{proof}

\begin{lemma}\label{lemma:same-marginal-correctness}
    When $\goodevent$ happens, the outputs of Algorithm~\ref{algo:same-marginal} satisfy $\err_{\D_i}(\hat f_i) \le (3 + \alpha)\eps$ for every $i \in [n]$.
\end{lemma}

\begin{proof}
    Fix $i \in [n]$, and consider the $i$-th iteration of the for-loop in Algorithm~\ref{algo:same-marginal}. If the condition $\err_S(\hat f) \le \left(3 + \frac{2}{3}\alpha\right)\eps$ holds, by the first condition in the definition of $\goodevent$, we must have $\err_{\D_i}(\hat f) \le (3 + \alpha)\eps$. Then, by setting $\hat f_i$ to $\hat f$, we guarantee that $\hat f_i$ is $(3 + \alpha)\eps$-accurate for $\D_i$. Otherwise, we pick $\hat f_i$ in Line~\ref{line:same-marginal-naive}, in which case the second condition of $\goodevent$ gives $\err_{\D_i}(\hat f_i) \le (1 + \alpha)\eps \le (3 + \alpha)\eps$.
\end{proof}

\begin{lemma}\label{lemma:same-marginal-efficiency}
    When $\goodevent$ happens, Algorithm~\ref{algo:same-marginal} runs in $\poly(n, d, 1/\eps, \log(1/\delta))$ time, makes at most $k$ calls to the ERM oracle, and takes
    \[
        O\left(\frac{kd\log(1/\eps)}{\eps} + \frac{n\log(n/\delta)}{\eps}\right)
    \]
    samples.
\end{lemma}

\begin{proof}
    The key of the proof is to show that when $\goodevent$ happens, $|F|$ is always at most $k$ throughout the execution of Algorithm~\ref{algo:same-marginal}.

    \paragraph{Upper bound $|F|$.} Suppose towards a contradiction that $|F| > k$ at the end of Algorithm~\ref{algo:same-marginal}, while $\goodevent$ happens. By definition of $(k, \eps)$-realizability from Definition~\ref{def:k-eps-realizable}, there exist $f^*_1, \ldots, f^*_k \in \F$ and $c \in [k]^n$ such that $\err_{\D_i}(f^*_{c_i}) \le \eps$ holds for every $i \in [n]$. By the pigeonhole principle, there exist $i < j$ such that $c_i = c_j$, and Algorithm~\ref{algo:same-marginal} increments $|F|$ on both the $i$-th and the $j$-th iterations of the for-loop.

    During the $i$-th iteration, we add $\hat f_i$ to $F$. By the second condition in the definition of $\goodevent$, we have $\err_{\D_i}(\hat f_i) \le (1 + \alpha / 3)\eps$. Now, we use the fact that $\D_i$ and $\D_j$ share the same marginal over $\X$, which we denote by $\D_x$. Define function $p_i: \X \to [0, 1]$ as $p_i(x') \coloneqq \pr{(x, y) \sim \D_i}{y = 1|x = x'}$, i.e., the expectation of $y|x$ according to $\D_i$. Define $p_j(x') \coloneqq \pr{(x, y)\sim \D_j}{y = 1|x = x'}$ analogously. Note that we have the following relation for every $f: \X \to \{0, 1\}$ and $i' \in \{i, j\}$:
    \[
        \err_{\D_{i'}}(f)
    =   \pr{(x, y) \sim \D_{i'}}{f(x) \ne y}
    =   \Ex{x \sim \D_x}{\pr{y \sim \Bern(p_{i'}(x))}{f(x) \ne y}}
    =   \Ex{x \sim \D_x}{|f(x) - p_{i'}(x)|}.
    \]
    Since $c_i = c_j$, for every $x \in \X$ we have
    \[
        \hat f_i(x) - p_j(x)
    =   [\hat f_i(x) - p_i(x)] + [p_i(x) - f^*_{c_i}(x)] + [f^*_{c_j}(x) - p_j(x)].
    \]
    It then follows from the triangle inequality that
    \begin{align*}
        \err_{\D_j}(\hat f_i)
    &=  \Ex{x \sim \D_x}{\left|\hat f_i(x) - p_j(x)\right|}\\
    &\le\Ex{x \sim \D_x}{\left|\hat f_i(x) - p_i(x)\right|} + \Ex{x \sim \D_x}{\left|p_i(x) - f^*_{c_i}(x)\right|} + \Ex{x \sim \D_x}{\left|f^*_{c_j}(x) - p_j(x)\right|}\\
    &=  \err_{\D_i}(\hat f_i) + \err_{\D_i}(f^*_{c_i}) + \err_{\D_j}(f^*_{c_j})\\
    &\le (3 + \alpha/3)\eps.
    \end{align*}
    The last step above applies $\err_{\D_i}(\hat f_i) \le (1 + \alpha / 3)\eps$, $\err_{\D_i}(f^*_{c_i}) \le \eps$, and $\err_{\D_j}(f^*_{c_j}) \le \eps$. The first inequality was proved earlier. The second and the third inequalities follow from our choice of $f^*_1, \ldots, f^*_k$ and $c \in [k]^n$.

    Finally, by the first condition in the definition of $\goodevent$, $\err_{\D_j}(\hat f_i) \le (3 + \alpha / 3)\eps$ implies that, during the $j$-th iteration of the for-loop, we have $\err_S(\hat f_i) \le \left(3 + \frac{2}{3}\alpha\right)\eps$ on Line~\ref{line:same-marginal-test}. Then, we will not increment $|F|$ during the $j$-th iteration, which leads to a contradiction.

    \paragraph{Oracle calls, runtime, and sample complexity.} We first note that $|F|$ is incremented each time we call the ERM oracle on Line~\ref{line:same-marginal-naive}. Therefore, we make at most $k$ calls to the ERM oracle. Other than this step, the remainder of Algorithm~\ref{algo:same-marginal} can clearly be implemented in polynomial time. Finally, to bound the sample complexity, we note that in every iteration of the for-loop, $c \cdot \frac{\ln(n|F|/\delta)}{\eps}
    =   O\left(\frac{\log(n/\delta)}{\eps}\right)$ samples are drawn. In addition, before each time $|F|$ is incremented, $O\left(\frac{d\log(1/\eps) + \log(k/\delta)}{\eps}\right)$ samples are drawn. Therefore, the total sample complexity is upper bounded by:
    \[
        n\cdot O\left(\frac{\log(n/\delta)}{\eps}\right) + k\cdot O\left(\frac{d\log(1/\eps) + \log(k/\delta)}{\eps}\right)
    =   O\left(\frac{kd\log(1/\eps) + n\log(n/\delta)}{\eps}\right).
    \]
\end{proof}

Finally, we put everything together to prove Theorem~\ref{thm:same-marginal}.
\begin{proof}[Proof of Theorem~\ref{thm:same-marginal}]
    By Lemmas \ref{lemma:same-marginal-good-event}, \ref{lemma:same-marginal-correctness}~and~\ref{lemma:same-marginal-efficiency}, conditioning on an event that happens with probability at least $1 - \delta$, Algorithm~\ref{algo:same-marginal} returns $(3 + \alpha)\eps$-accurate classifiers for all the $n$ distributions, and the runtime, number of ERM oracle calls, and the number of samples are bounded accordingly.

    In order to control the (unconditional) sample complexity, runtime, and number of oracle calls, we simply terminate the algorithm when any of these quantities exceeds the corresponding bound. The resulting algorithm is still $((3 + \alpha)\eps, \delta)$-PAC, and satisfies the desired upper bounds on the sample complexity, runtime, and number of oracle calls.
\end{proof}

\section{Efficient Learning via Approximate Coloring} \label{sec:approx-coloring}
In this section, we prove Theorem~\ref{thm:refutable}, which gives efficient learning algorithms when the hypothesis class is $2$-refutable. 

\subsection{The Learning Algorithm}
The two cases are proved via a common strategy. In each iteration, we carefully choose a parameter $m$ and draw $m$ samples from each distribution. We use an approximate coloring algorithm to color the conflict graph (Definition~\ref{def:conflict-graph}) induced by the datasets. For each color that is used by considerably many vertices, we combine the corresponding datasets and fit a classifier to this joint dataset. The key is to argue that this classifier must be accurate for many distributions. Finally, we repeat the above on the distributions that have not received an accurate classifier.

We formally define a meta-algorithm in Algorithm~\ref{algo:refutable}. In the $r$-th iteration of the while-loop, we draw $m^{(r)}$ samples from each of the remaining distributions in $G^{(r)}$. We then build the conflict graph based on these datasets, and compute a $\gamma^{(r)}$-coloring of the graph. The vertices that receive color $i$ are denoted by $G_i$, and $\hat g_i$ is chosen as an arbitrary classifier in $\F$ that is consistent with $S_v$ for every $v \in G_i$. This choice is always possible by Lemma~\ref{lemma:conflit-graph-IS}.

The algorithm is under-specified in three aspects: the number of samples $m^{(r)}$, the number of colors $\gamma^{(r)}$, as well as the algorithm for computing a $\gamma^{(r)}$-coloring. We will specify these choices when we prove Theorem~\ref{thm:refutable} later.

\begin{algorithm}[ht]
    \caption{Collaborative Learning via Approximate Coloring}
    \label{algo:refutable}
    \begin{algorithmic}[1]
    \STATE \textbf{Input:} $2$-refutable hypothesis class $\F$. Sample access to $\D_1, \ldots, \D_n$. Parameters $k, \eps, \delta, c$.
    \STATE \textbf{Output:} Hypotheses $\hat f_1, \hat f_2, \ldots, \hat f_n$.
    \STATE $r \gets 1$; $G^{(1)} \gets [n]$;
    \WHILE{$G^{(r)} \ne \emptyset$}
        \STATE $\delta^{(r)} \gets \delta / r^2$;
        \STATE Set parameters $m^{(r)}$ and $\gamma^{(r)}$ according to $|G^{(r)}|$, $d$, $\eps$, $\delta^{(r)}$;
        \FOR{$i \in G^{(r)}$}
            \STATE Draw $m^{(r)}$ samples from $\D_i$ to form $S_i$;
        \ENDFOR
        \STATE $(G^{(r)}, E) \gets$ conflict graph of $\{S_i: i \in G^{(r)}\}$;
        \STATE Compute a $\gamma^{(r)}$-coloring of $(G^{(r)}, E)$. Let $G_i \subseteq G^{(r)}$ denote the set of vertices with color $i$;
        \STATE $G^{(r+1)} \gets G^{(r)}$;
        \FOR{$i \in [\gamma^{(r)}]$ such that $|G_i| \ge |G^{(r)}| / (2\gamma^{(r)})$}
            \STATE Find $\hat g_i \in \F$ such that $\err_{S_v}(\hat g_i) = 0$, $\forall v \in G_i$; \alglinelabel{line:conflict-graph-IS}
            \FOR{$v \in G_i$}
                \STATE Draw $c \cdot \frac{\ln(|G^{(r)}|/\delta^{(r)})}{\eps}$ samples from $\D_v$ to form $S_v$; \alglinelabel{line:refutable-test}
                \IF{$\err_{S_v}(\hat g_i) \le \eps / 2$}
                    \STATE $\hat f_v \gets \hat g_i$;
                    \STATE $G^{(r+1)} \gets G^{(r+1)} \setminus \{v\}$;
                \ENDIF
            \ENDFOR
        \ENDFOR
        \STATE $r \gets r + 1$;
    \ENDWHILE
    \STATE \textbf{Return} $\hat f_1, \hat f_2, \ldots, \hat f_n$;
    \end{algorithmic}
\end{algorithm}

\subsection{Technical Lemmas}

We state and prove a few technical lemmas for the analysis of Algorithm~\ref{algo:refutable}. The key of the analysis is the following lemma, which states that, as long as a sufficiently many vertices are colored with color $i$, the learned classifier $\hat g_i$ is good on average for the vertices with color $i$.
\begin{lemma}\label{lemma:single-iteration}
    There is a universal constant $c > 0$ such that the following holds. In the $r$-th iteration of the while-loop, if $m^{(r)}$ is at least
    \[
        c\cdot\max\left\{\frac{\gamma^{(r)}}{|G^{(r)}|}\cdot \frac{d\ln(1/\eps) + |G^{(r)}| + \ln(1/\delta^{(r)})}{\eps}, \ln\frac{|G^{(r)}|}{\delta^{(r)}}\right\},
    \]
    it holds with probability $1 - \delta^{(r)} / 6$ that, for every $i \in [\gamma^{(r)}]$ such that $|G_i| \ge |G^{(r)}| / (2\gamma^{(r)})$,
    \[
        \frac{1}{|G_i|}\sum_{v\in G_i}\err_{\D_v}(\hat g_i) \le \eps / 8.
    \]
\end{lemma}

The proof is based on similar techniques to the analysis of~\cite{Qiao18} for a different variant of collaborative learning, in which a small fraction of the data sources are adversarial.

\begin{proof}
    For brevity, we omit the superscripts in $m^{(r)}$, $\delta^{(r)}$ and $\gamma^{(r)}$. Let $M \coloneqq m \cdot \frac{|G^{(r)}|}{4\gamma}$.
    Consider the following thought experiment: We draw $M$ independent samples $z^{(i)}_1, z^{(i)}_2, \ldots, z^{(i)}_{M}$ from each $\D_i$. (Recall that in Algorithm~\ref{algo:refutable}, only $m \ll M$ data points are actually drawn to form the dataset $S_i$.) Independently, for each non-empty $U \subseteq G^{(r)}$, we choose a sequence $A^{(U)} \in U^M$ uniformly at random. We may then consider the \emph{fictitious dataset} $S^{(U)} = \left\{S^{(U)}_1, \ldots, S^{(U)}_M\right\}$ defined as:
    \[
        S^{(U)}_i
    \coloneqq
        z^{(j)}_k,
    \text{ where }
    j = A^{(U)}_i,
    k = \sum_{l=1}^{i}\1{A^{(U)}_l = j}.
    \]
    In words, for each $i \in U$, if entry $i$ appears $t$ times in sequence $A^{(U)}$, $S^{(U)}$ contains the first $t$ data points collected from $\D_i$ (namely, $z^{(i)}_1, \ldots, z^{(i)}_t$). It can be easily verified that, over the randomness in all $z^{(i)}$ and $A^{(U)}$, each fictitious dataset $S^{(U)}$ is identically distributed as $M$ samples from the uniform mixture $\D_U \coloneqq \frac{1}{|U|}\sum_{i \in U}\D_i$.

    The rest of the proof consists of two parts: First, we show that with high probability, every $S^{(U)}$ is ``representative'' for distribution $\D_U$. Formally, any classifier in $\F$ with a zero training error on $S^{(U)}$ must have an $O(\eps)$ population error on $\D_U$. Then, we show that for each color $i \in [\gamma]$, the actual datasets (each of size $m$) collected from the distributions with color $i$ \emph{can} simulate the fictitious dataset $S^{(G_i)}$.

    \paragraph{Step 1: Fictitious datasets are representative.}
    For each fixed non-empty $U \subseteq G^{(r)}$, Theorems 28.3~and~28.4 of~\cite{SSBD14} imply that for some universal constant $c' > 0$,
    \[
        \pr{}{\forall f \in \F, \err_{S^{(U)}}(f) = 0 \implies \err_{\D_U}(f) \le \eps / 8}
    \ge 1 - (8/\eps)^d\cdot e^{-\eps M / c'}.
    \]
    For $M \ge c' \cdot \frac{d\ln(8/\eps) + |G^{(r)}|\ln 2 + \ln(12/\delta)}{\eps}$, the right-hand side above is at least $1 - \frac{\delta}{12\cdot 2^{|G^{(r)}|}}$. Then, a union bound over the $2^{|G^{(r)}|} - 1$ choices of $U$ shows that with probability at least $1 - \delta / 12$, for all non-empty $U \subseteq G^{(r)}$, any classifier in $\F$ that is consistent with $S^{(U)}$ has an error $\le \eps / 8$ on distribution $\D_U$.

    \paragraph{Step 2: Fictitious datasets can be simulated.} Fix $i \in [\gamma]$ such that $|G_i| \ge |G^{(r)}| / (2\gamma)$. Recall that the classifier $\hat g_i$ has a zero training error on $T_i \coloneqq \bigcup_{v \in G_i}S_v$. In the first step, we showed that any $f \in \F$ that achieves a zero training error on $S^{(G_i)}$ must have a small population error on $\D_{G_i}$. Thus, it suffices to argue that $T_i \supseteq S^{(G_i)}$ with high probability.
    
    Recall that we computed the coloring solely based on the datasets, which are independent of the indices $A^{(U)}$. Therefore, conditioning on the realization of $G_1, G_2, \ldots, G_{\gamma}$, each $A^{(G_i)}$ still uniformly distributed among $G_i^{M}$. In particular, for every $i \in [\gamma]$ and $v \in G_i$, the number of times $v$ appears in $A^{(G_i)}$, denoted by $n_{i,v}$, follows the binomial distribution $\Bin(M, 1 / |G_i|)$. As long as $n_{i,v} \le m$ for every $(i, v)$ pair, each $T_i$ (which contains the first $m$ data points from $\D_v$) will be a superset of $S^{(G_i)}$ (which contains the first $n_{i,v}$ data points from $\D_v$).

    There are at most $|G^{(r)}|$ such $(i, v)$ pairs. The probability for each pair to violate the condition is at most
    \begin{align*}
        \pr{X \sim \Bin(M, 1 / |G_i|)}{X \ge m}
    &\le\pr{X \sim \Bin(M, 1 / |G_i|)}{X \ge \frac{4M\gamma}{|G^{(r)}|}} \tag{$M = m|G^{(r)}|/(4\gamma)$}\\
    &\le\pr{X \sim \Bin(M, 2\gamma / |G^{(r)}|)}{X \ge \frac{4M\gamma}{|G^{(r)}|}}. \tag{$|G_i| \ge |G^{(r)}| / (2\gamma)$}
    \end{align*}
    By a Chernoff bound, the last expression is at most $\exp\left(-\frac{2M\gamma}{3|G^{(r)}|}\right)$, which can be made smaller than $\frac{\delta}{12|G^{(r)}|}$ since $M \ge c \cdot \frac{|G^{(r)}|}{4\gamma}\ln\frac{|G^{(r)}|}{\delta}$ for sufficiently large $c$.
    By a union bound, the aforementioned condition holds for all $(i, v)$ pairs with probability at least $1 - \delta / 12$.

    Finally, the lemma follows from the two steps above and another union bound.
\end{proof}

As in the analysis in the previous section, we define a ``good event'' that implies the success of Algorithm~\ref{algo:refutable}.

\begin{definition}
    Let $\goodevent$ denote the event that the following happen simultaneously when Algorithm~\ref{algo:refutable} is executed:
    \begin{itemize}
        \item The condition in Lemma~\ref{lemma:single-iteration} holds at every iteration $r$.
        \item Whenever Line~\ref{line:refutable-test} is reached, $\err_{\D_v}(\hat g_i) \le \eps / 4$ implies $\err_{S_v}(\hat g_i) \le \eps / 2$ and $\err_{\D_v}(\hat g_i) > \eps$ implies $\err_{S_v}(\hat g_i) > \eps / 2$.
    \end{itemize}
\end{definition}

\begin{lemma}\label{lemma:refutable-good-event}
    When Algorithm~\ref{algo:refutable} is executed with some sufficiently large constant $c$, it holds that $\pr{}{\goodevent} \ge 1 - \delta$.
\end{lemma}
\begin{proof}
    By Lemma~\ref{lemma:single-iteration}, the probability for the condition in Lemma~\ref{lemma:single-iteration} to be violated in the $r$-th iteration is at most $\delta^{(r)} / 6$. Summing over all $r$ gives $\sum_{r=1}^{+\infty}\frac{\delta^{(r)}}{6}
    =   \frac{\delta}{6}\cdot \frac{\pi^2}{6}
    <   \delta / 3$.
    By the same argument as in the proof of Lemma~\ref{lemma:same-marginal-good-event}, the probability for the second condition to be violated is also at most $\delta/3$. By a union bound, $\pr{}{\goodevent} \ge 1 - \delta/3 - \delta/3 \ge 1 - \delta$.
\end{proof}

Analogous to Lemma~\ref{lemma:same-marginal-correctness}, we have the following lemma, which states that event $\goodevent$ guarantees that the classifiers returned by the algorithm are accurate.
\begin{lemma}\label{lemma:refutable-correctness}
    When $\goodevent$ happens, the output of Algorithm~\ref{algo:refutable} satisfies $\err_{\D_i}(\hat f_i) \le \eps$ for every $i \in [n]$.
\end{lemma}

Finally, we prove that the number of active distributions, $|G^{(r)}|$, decreases at an exponential rate, so the while-loop is executed at most $O(\log n)$ times. This will be useful for upper bounding the sample complexity.

\begin{lemma}\label{lemma:refutable-efficiency}
    When event $\goodevent$ happens, $|G^{(r+1)}| \le \frac{3}{4} |G^{(r)}|$ holds at the end of the $r$-th iteration of the while-loop.
\end{lemma}

\begin{proof}
    Consider the $r$-th iteration of the while-loop. For brevity, we drop the superscript in $\gamma^{(r)}$. Since $\sum_{i=1}^{\gamma}|G_i| = |G^{(r)}|$, we have
    \begin{align*}
        &~\sum_{i=1}^{\gamma}|G_i|\cdot\1{|G_i| \ge |G^{(r)}| / (2\gamma)}\\
    =   &~\sum_{i=1}^{\gamma}|G_i| - \sum_{i=1}^{\gamma}|G_i|\cdot\1{|G_i| < |G^{(r)}| / (2\gamma)}\\
    \ge &~ |G^{(r)}| - \gamma\cdot\frac{|G^{(r)}|}{2\gamma}
    =   |G^{(r)}| / 2.
    \end{align*}
    
    Fix $i \in [\gamma]$ that satisfies $|G_i| \ge |G^{(r)}| / (2\gamma)$. By the first condition in the definition of $\goodevent$, event $\goodevent$ implies that
    \[
        \frac{1}{|G_i|}\sum_{v \in G_i}\err_{\D_v}(\hat g_i) \le \eps / 8.
    \]
    By Markov's inequality, there are at least $|G_i| / 2$ elements $v \in G_i$ such that $\err_{\D_v}(\hat g_i) \le \eps / 4$. Then, by the second condition in the definition of $\goodevent$, every such element $v$ will not appear in $G^{(r + 1)}$. Therefore, we conclude that
    \[
        |G^{(r+1)}|
    \le |G^{(r)}| - \sum_{i=1}^{\gamma}\frac{|G_i|}{2}\cdot\1{|G_i| \ge |G^{(r)}| / (2\gamma)}
    \le |G^{(r)}| - \frac{|G^{(r)}|}{4}
    =   \frac{3}{4}|G^{(r)}|.
    \]
\end{proof}

\subsection{The Bipartite Case}
We start with the simpler case that $k = 2$. In this case, we set $m^{(r)}$ according to Lemma~\ref{lemma:single-iteration} and set $\gamma^{(r)} = 2$ in Algorithm~\ref{algo:refutable}. Furthermore, the coloring algorithm is simply the efficient algorithm for $2$-coloring.

\begin{proof}[Proof of Theorem~\ref{thm:refutable}, the $k = 2$ case]
    In light of Lemmas \ref{lemma:refutable-good-event}~and~\ref{lemma:refutable-correctness}, it remains to upper bound the sample complexity of Algorithm~\ref{algo:refutable} under event $\goodevent$. As a simple corollary of Lemma~\ref{lemma:refutable-efficiency}, we have $|G^{(r)}| \le (3/4)^{r-1}\cdot n$ at the $r$-th iteration of the while-loop.

    The sample complexity of the $r$-th iteration of the while-loop is upper bounded by
    \begin{align*}
        &~m^{(r)}\cdot|G^{(r)}| + |G^{(r)}|\cdot c\cdot \frac{\ln(|G^{(r)}| / \delta^{(r)})}{\eps}\\
    \preceq
        &~\frac{d\log(1/\eps) + |G^{(r)}| + \log(1/\delta^{(r)})}{\eps} + |G^{(r)}|\cdot \frac{\log(|G^{(r)}| / \delta^{(r)})}{\eps}\\
    \preceq
        &~ \frac{d\log(1/\eps) + (3/4)^r\cdot n + \log(1/\delta) + \log r}{\eps} + (3/4)^r\cdot n\cdot\frac{\log[(3/4)^r\cdot n] + \log(1/\delta) + \log r}{\eps}.
    \end{align*}
    Since the while-loop terminates when $G^{(r)}$ is empty, there are at most $O(\log n)$ iterations, and summing the above over all rounds gives
    \begin{align*}
        &~\frac{d\log(1/\eps)\log n + n + \log(1/\delta)\log n + \log^2n}{\eps} + \frac{n\log n + n \log (1/\delta)}{\eps}\\
    \preceq &~\frac{d\log(1/\eps) + n}{\eps}\cdot \log n + \frac{n \log(1/\delta)}{\eps}.
    \end{align*}
    Therefore, we have the desired sample complexity upper bound.
\end{proof}

\subsection{The General Case}
When $k \ge 3$, we can no longer find a $k$-coloring efficiently. Instead, we compute an approximate coloring with $O(n^{c^*_k})$ colors, where $n = |G^{(r)}|$ is the number of vertices in the graph. The hope is that as long as $c^*_k < 1$, we can still combine the datasets from vertices that share the same color, and use the data more efficiently.

Formally, let $\alpha$ be a constant such that there is an efficient algorithm that colors every $k$-colorable graph with $n$ vertices using at most $\alpha \cdot n^{c^*_k}$ colors. We set $\gamma^{(r)} = \alpha \cdot |G^{(r)}|^{c^*_k}$ and set $m^{(r)}$ according to Lemma~\ref{lemma:single-iteration}.

\begin{proof}[Proof of Theorem~\ref{thm:refutable}, the $k \ge 3$ case]
    Again, we focus on upper bounding the sample complexity. The number of samples drawn in the $r$-th round is at most
    \begin{align*}
        &~m^{(r)}\cdot|G^{(r)}| + |G^{(r)}|\cdot c\cdot \frac{\ln(|G^{(r)}| / \delta^{(r)})}{\eps}\\
    \preceq
        &~|G^{(r)}|^{c^*_k}\cdot\frac{d\log(1/\eps) + |G^{(r)}| + \log(1/\delta^{(r)})}{\eps} + |G^{(r)}|\cdot \frac{\log(|G^{(r)}| / \delta^{(r)})}{\eps}.
    \end{align*}
    Plugging $|G^{(r)}| \le (3/4)^{r-1}\cdot n$ into the above and summing over $r = 1, 2, \ldots$ gives
    \begin{align*}
        &~\frac{dn^{c^*_k}\log(1/\eps) + n^{1+c^*_k} + n^{c^*_k}\log(1/\delta)}{\eps} + \frac{n\log n + n \log(1/\delta)}{\eps}\\
    \preceq
        &~\frac{d\log(1/\eps) + n}{\eps}\cdot n^{c^*_k} + \frac{n\log(1/\delta)}{\eps}. \qedhere
    \end{align*}
\end{proof}

\bibliographystyle{alpha}
\bibliography{references}

\newcommand{\etalchar}[1]{$^{#1}$}
\begin{thebibliography}{BDBC{\etalchar{+}}10}

\bibitem[AB99]{AB99}
Martin Anthony and Peter~L. Bartlett.
\newblock {\em Neural Network Learning: Theoretical Foundations}.
\newblock Cambridge University Press, 1999.

\bibitem[ACC06]{ACC06}
Sanjeev Arora, Eden Chlamtac, and Moses Charikar.
\newblock New approximation guarantee for chromatic number.
\newblock In {\em Symposium on Theory of Computing (STOC)}, pages 215--224,
  2006.

\bibitem[AHZ23]{AHZ23}
Pranjal Awasthi, Nika Haghtalab, and Eric Zhao.
\newblock Open problem: The sample complexity of multi-distribution learning
  for vc classes.
\newblock In {\em Conference on Learning Theory (COLT)}, pages 5943--5949,
  2023.

\bibitem[BBV15]{BBV15}
Maria-Florina Balcan, Avrim Blum, and Santosh Vempala.
\newblock Efficient representations for lifelong learning and autoencoding.
\newblock In {\em Conference on Learning Theory (COLT)}, pages 191--210, 2015.

\bibitem[BDBC{\etalchar{+}}10]{ben2010theory}
Shai Ben-David, John Blitzer, Koby Crammer, Alex Kulesza, Fernando Pereira, and
  Jennifer~Wortman Vaughan.
\newblock A theory of learning from different domains.
\newblock {\em Machine learning}, 79:151--175, 2010.

\bibitem[BHPQ17]{BHPQ17}
Avrim Blum, Nika Haghtalab, Ariel~D Procaccia, and Mingda Qiao.
\newblock Collaborative {PAC} learning.
\newblock In {\em Advances in Neural Information Processing Systems (NIPS)},
  pages 2389--2398, 2017.

\bibitem[BK97]{BK97}
Avrim Blum and David Karger.
\newblock An Õ($n^{3/14}$)-coloring algorithm for 3-colorable graphs.
\newblock {\em Information Processing Letters}, 61(1):49--53, 1997.

\bibitem[Blu94]{Blum94}
Avrim Blum.
\newblock New approximation algorithms for graph coloring.
\newblock {\em Journal of the ACM (JACM)}, 41(3):470--516, 1994.

\bibitem[BR90]{BR90}
Bonnie Berger and John Rompel.
\newblock A better performance guarantee for approximate graph coloring.
\newblock {\em Algorithmica}, 5(1-4):459--466, 1990.

\bibitem[CCD23]{cheng2023federated}
Gary Cheng, Karan Chadha, and John Duchi.
\newblock Federated asymptotics: a model to compare federated learning
  algorithms.
\newblock In {\em International Conference on Artificial Intelligence and
  Statistics (AISTATS)}, pages 10650--10689, 2023.

\bibitem[Chl07]{Chlamtac07}
Eden Chlamtac.
\newblock Approximation algorithms using hierarchies of semidefinite
  programming relaxations.
\newblock In {\em Foundations of Computer Science (FOCS)}, pages 691--701,
  2007.

\bibitem[CM12]{CM12}
Koby Crammer and Yishay Mansour.
\newblock Learning multiple tasks using shared hypotheses.
\newblock In {\em Advances in Neural Information Processing Systems (NIPS)},
  pages 1475--1483, 2012.

\bibitem[CZLS23]{JMLR:v24:21-0224}
Shuxiao Chen, Qinqing Zheng, Qi~Long, and Weijie~J. Su.
\newblock Minimax estimation for personalized federated learning: An
  alternative between fedavg and local training?
\newblock {\em Journal of Machine Learning Research}, 24(262):1--59, 2023.

\bibitem[CZZ18]{chen2018tight}
Jiecao Chen, Qin Zhang, and Yuan Zhou.
\newblock Tight bounds for collaborative pac learning via multiplicative
  weights.
\newblock In {\em Advances in Neural Information Processing Systems (NeurIPS)},
  pages 3602--3611, 2018.

\bibitem[DJKS23]{DJKS23}
Abhimanyu Das, Ayush Jain, Weihao Kong, and Rajat Sen.
\newblock Efficient list-decodable regression using batches.
\newblock In {\em International Conference on Machine Learning (ICML)}, pages
  7025--7065, 2023.

\bibitem[DKMR22]{DKMR22}
Ilias Diakonikolas, Daniel Kane, Pasin Manurangsi, and Lisheng Ren.
\newblock Cryptographic hardness of learning halfspaces with massart noise.
\newblock {\em Advances in Neural Information Processing Systems (NeurIPS)},
  35:3624--3636, 2022.

\bibitem[DKR23]{DKR23}
Ilias Diakonikolas, Daniel Kane, and Lisheng Ren.
\newblock Near-optimal cryptographic hardness of agnostically learning
  halfspaces and relu regression under gaussian marginals.
\newblock In {\em International Conference on Machine Learning (ICML)}, pages
  7922--7938, 2023.

\bibitem[Fel06]{Feldman06}
Vitaly Feldman.
\newblock Optimal hardness results for maximizing agreements with monomials.
\newblock In {\em Conference on Computational Complexity (CCC)}, pages
  226--236, 2006.

\bibitem[FGKP06]{FGKP06}
Vitaly Feldman, Parikshit Gopalan, Subhash Khot, and Ashok~Kumar Ponnuswami.
\newblock New results for learning noisy parities and halfspaces.
\newblock In {\em Foundations of Computer Science (FOCS)}, pages 563--574,
  2006.

\bibitem[GR09]{GR09}
Venkatesan Guruswami and Prasad Raghavendra.
\newblock Hardness of learning halfspaces with noise.
\newblock {\em SIAM Journal on Computing}, 39(2):742--765, 2009.

\bibitem[HJZ22]{haghtalab2022demand}
Nika Haghtalab, Michael Jordan, and Eric Zhao.
\newblock On-demand sampling: Learning optimally from multiple distributions.
\newblock In {\em Advances in Neural Information Processing Systems (NeurIPS)},
  pages 406--419, 2022.

\bibitem[HK22]{hanneke2022no}
Steve Hanneke and Samory Kpotufe.
\newblock A no-free-lunch theorem for multitask learning.
\newblock {\em The Annals of Statistics}, 50(6):3119--3143, 2022.

\bibitem[HXL{\etalchar{+}}23]{huang2023optimal}
Xinmeng Huang, Kan Xu, Donghwan Lee, Hamed Hassani, Hamsa Bastani, and Edgar
  Dobriban.
\newblock Optimal heterogeneous collaborative linear regression and contextual
  bandits.
\newblock {\em arXiv preprint arXiv:2306.06291}, 2023.

\bibitem[JSK{\etalchar{+}}23]{JSKDO23}
Ayush Jain, Rajat Sen, Weihao Kong, Abhimanyu Das, and Alon Orlitsky.
\newblock Linear regression using heterogeneous data batches.
\newblock {\em arXiv preprint arXiv:2309.01973}, 2023.

\bibitem[KL19]{konstantinov2019robust}
Nikola Konstantinov and Christoph Lampert.
\newblock Robust learning from untrusted sources.
\newblock In {\em International Conference on Machine Learning (ICML)}, pages
  3488--3498, 2019.

\bibitem[KMS98]{KMS98}
David Karger, Rajeev Motwani, and Madhu Sudan.
\newblock Approximate graph coloring by semidefinite programming.
\newblock {\em Journal of the ACM (JACM)}, 45(2):246--265, 1998.

\bibitem[KSKO20]{KSKO20}
Weihao Kong, Raghav Somani, Sham Kakade, and Sewoong Oh.
\newblock Robust meta-learning for mixed linear regression with small batches.
\newblock In {\em Advances in Neural Information Processing Systems (NeurIPS)},
  pages 4683--4696, 2020.

\bibitem[KSS92]{KSS92}
Michael~J Kearns, Robert~E Schapire, and Linda~M Sellie.
\newblock Toward efficient agnostic learning.
\newblock In {\em Annual Workshop on Computational Learning Theory (COLT)},
  pages 341--352, 1992.

\bibitem[KSS{\etalchar{+}}20]{KSSKO20}
Weihao Kong, Raghav Somani, Zhao Song, Sham Kakade, and Sewoong Oh.
\newblock Meta-learning for mixed linear regression.
\newblock In {\em International Conference on Machine Learning (ICML)}, pages
  5394--5404, 2020.

\bibitem[KST23a]{KST23a}
Caleb Koch, Carmen Strassle, and Li-Yang Tan.
\newblock Properly learning decision trees with queries is np-hard.
\newblock In {\em Foundations of Computer Science (FOCS)}, pages 2383--2407,
  2023.

\bibitem[KST23b]{KST23b}
Caleb Koch, Carmen Strassle, and Li-Yang Tan.
\newblock Superpolynomial lower bounds for decision tree learning and testing.
\newblock In {\em Symposium on Discrete Algorithms (SODA)}, pages 1962--1994,
  2023.

\bibitem[KT17]{KT17}
Ken-Ichi Kawarabayashi and Mikkel Thorup.
\newblock Coloring 3-colorable graphs with less than n1/5 colors.
\newblock {\em Journal of the ACM (JACM)}, 64(1):1--23, 2017.

\bibitem[MMR08]{mansour2008domain}
Yishay Mansour, Mehryar Mohri, and Afshin Rostamizadeh.
\newblock Domain adaptation with multiple sources.
\newblock In {\em Advances in Neural Information Processing Systems (NIPS)},
  pages 1041--1048, 2008.

\bibitem[MMR{\etalchar{+}}17]{mcmahan2017communication}
Brendan McMahan, Eider Moore, Daniel Ramage, Seth Hampson, and Blaise~Aguera
  y~Arcas.
\newblock Communication-efficient learning of deep networks from decentralized
  data.
\newblock In {\em International Conference on Artificial Intelligence and
  Statistics (AISTATS)}, pages 1273--1282, 2017.

\bibitem[MMR{\etalchar{+}}21]{mansour2021theory}
Yishay Mansour, Mehryar Mohri, Jae Ro, Ananda~Theertha Suresh, and Ke~Wu.
\newblock A theory of multiple-source adaptation with limited target labeled
  data.
\newblock In {\em International Conference on Artificial Intelligence and
  Statistics (AISTATS)}, pages 2332--2340, 2021.

\bibitem[MSS19]{mohri2019agnostic}
Mehryar Mohri, Gary Sivek, and Ananda~Theertha Suresh.
\newblock Agnostic federated learning.
\newblock In {\em International Conference on Machine Learning (ICML)}, pages
  4615--4625, 2019.

\bibitem[NZ18]{nguyen2018improved}
Huy Nguyen and Lydia Zakynthinou.
\newblock Improved algorithms for collaborative pac learning.
\newblock In {\em Advances in Neural Information Processing Systems (NeurIPS)},
  pages 7631--7639, 2018.

\bibitem[Pen23]{Peng23}
Binghui Peng.
\newblock The sample complexity of multi-distribution learning.
\newblock {\em arXiv preprint arXiv:2312.04027}, 2023.

\bibitem[PU16]{PU16}
Anastasia Pentina and Ruth Urner.
\newblock Lifelong learning with weighted majority votes.
\newblock In {\em Advances in Neural Information Processing Systems (NIPS)},
  pages 3619--3627, 2016.

\bibitem[Qia18]{Qiao18}
Mingda Qiao.
\newblock Do outliers ruin collaboration?
\newblock In {\em International Conference on Machine Learning (ICML)}, pages
  4180--4187, 2018.

\bibitem[SSBD14]{SSBD14}
Shai Shalev-Shwartz and Shai Ben-David.
\newblock {\em Understanding machine learning: From theory to algorithms}.
\newblock Cambridge university press, 2014.

\bibitem[Wig83]{Wigderson83}
Avi Wigderson.
\newblock Improving the performance guarantee for approximate graph coloring.
\newblock {\em Journal of the ACM (JACM)}, 30(4):729--735, 1983.

\bibitem[YZW{\etalchar{+}}20]{yang2020analysis}
Fan Yang, Hongyang~R. Zhang, Sen Wu, Christopher R{\'e}, and Weijie~J. Su.
\newblock Precise high-dimensional asymptotics for quantifying heterogeneous
  transfers.
\newblock {\em arXiv preprint arXiv:2010.11750}, 2020.

\bibitem[ZZC{\etalchar{+}}23]{zhang2023optimal}
Zihan Zhang, Wenhao Zhan, Yuxin Chen, Simon~S Du, and Jason~D Lee.
\newblock Optimal multi-distribution learning.
\newblock {\em arXiv preprint arXiv:2312.05134}, 2023.

\end{thebibliography}

\newpage
\appendix

\section{A Sample-Efficient Learning Algorithm}\label{sec:upper-proofs}
The algorithm is formally defined in Algorithm~\ref{algo:general}, and follows the same strategy as Algorithm~1 of~\cite{BHPQ17}.

\begin{algorithm}[ht]
    \caption{Collaborative Learning for $(k, \eps)$-Realizable Distributions}
    \label{algo:general}
    \begin{algorithmic}[1] % Line number every 1 line
    \STATE \textbf{Input:} Hypothesis class $\F$. Sample access to $\D_1$, $\ldots$, $\D_n$. Parameters $k, \eps, \delta, c$.
    \STATE \textbf{Output:} Hypotheses $\hat f_1, \hat f_2, \ldots, \hat f_n$.
    \STATE $r \gets 1$; $G^{(1)} \gets [n]$;
    \WHILE{$|G^{(r)}| > k$}
        \STATE $\delta^{(r)} \gets \delta / r^2$;
        \STATE $d^{(r)} \gets c \cdot (kd + |G^{(r)}|\log k)$;
        \STATE Draw $c\cdot\frac{d^{(r)}\ln(1/\eps) + \ln(1/\delta^{(r)})}{\eps}$ samples from $\overline{D} \coloneqq \frac{1}{|G^{(r)}|}\sum_{i \in G^{(r)}}\D'_i$ to form $S$; \alglinelabel{line:general-mixture}
        \STATE $g_{f, c} \gets \argmin_{g \in \F_{G^{(r)},k}}\err_{S}(g)$; \alglinelabel{line:general-ERM}
        \STATE $G^{(r+1)} \gets \emptyset$;
        \FOR{$i \in G^{(r)}$}
            \STATE Draw $c \cdot \frac{\ln(|G^{(r)}|/\delta^{(r)})}{\eps}$ samples from $\D_i$ to form $S_i$; \alglinelabel{line:general-test}
            \IF{$\err_{S_i}(f_{c_i}) \le 6\eps$}
                \STATE $\hat f_i \gets f_{c_i}$;
            \ELSE
                \STATE $G^{(r+1)} \gets G^{(r+1)} \cup \{i\}$;
            \ENDIF
        \ENDFOR
        \STATE $r \gets r + 1$;
    \ENDWHILE

    \FOR{$i \in G^{(r)}$}
        \STATE Draw $c \cdot \frac{d\ln(1/\eps) + \ln(k/\delta)}{\eps}$ samples from $\D_i$ to form $S_i$;
        \STATE $\hat f_i \gets \argmin_{f \in \F}\err_{S_i}(f)$; \alglinelabel{line:general-naive-ERM}
    \ENDFOR

    \STATE \textbf{Return:} $\hat f_1, \ldots, \hat f_n$;
    \end{algorithmic}
\end{algorithm}

Recall that for each data distribution $\D_i$, $\D'_i$ denotes the distribution of $((i, x), y)$ when $(x, y) \sim \D_i$. Therefore, on Line~\ref{line:general-mixture}, to sample from the mixture distribution $\frac{1}{|G^{(r)}|}\sum_{i \in G^{(r)}}\D'_i$, it suffices to sample $i$ from $G^{(r)}$ uniformly at random, draw $(x, y) \sim \D_i$, and then use the labeled example $((i, x), y)$.

The algorithm maintains $G^{(r)}$ as the set of active distributions at the beginning of the $r$-th round. The algorithm samples from the uniform mixture of the active distributions, learns a classifier $g_{f, c} \in \F_{G^{(r)},k}$ via ERM, and then tests whether the learned classifier is good enough for each distribution in $G^{(r)}$. If the learned classifier achieves an $O(\eps)$ empirical error on $\D_i$, we use it as the answer $\hat f_i$; otherwise, $\D_i$ stays active for the next round. Finally, the iteration terminates whenever the number of active distributions drops below $k$, at which point we na\"ively learn on the $\le k$ remaining distributions separately.

The analysis of Algorithm~\ref{algo:general} is straightforward, and relies on the following definition of a ``good event''.

\begin{definition}
    Let $\goodevent$ denote the event that the following happen simultaneously when Algorithm~\ref{algo:general} is executed:
    \begin{itemize}
        \item Whenever Line~\ref{line:general-ERM} is reached, it holds that $\err_{\overline{\D}}(g_{f, c}) \le 2\eps$.
        \item Whenever Line~\ref{line:general-test} is reached, it holds that: (1)  $\err_{\D_i}(f_{c_i}) \le 4\eps$ implies $\err_{S_i}(f_{c_i}) \le 6\eps$; (2) $\err_{\D_i}(f_{c_i}) > 8\eps$ implies $\err_{S_i}(f_{c_i}) > 6\eps$.
        \item Whenever Line~\ref{line:general-naive-ERM} is reached, it holds that $\err_{\D_i}(\hat f_i) \le 2\eps$.
    \end{itemize}
\end{definition}

Then, Theorem~\ref{thm:sample-upper-general} is a consequence of the following three lemmas.

\begin{lemma}\label{lemma:general-good-event}
    For some universal constant $c$, when Algorithm~\ref{algo:general} is executed with parameter $c$,
    $\pr{}{\goodevent} \ge 1 - \delta$.
\end{lemma}

\begin{proof}
    Whenever Line~\ref{line:general-ERM} is reached, by Lemma~\ref{lemma:vc-dim-bound}, for some sufficiently large constant $c > 0$, the VC dimension of $\F_{G^{(r)}, k}$ is upper bounded by $d^{(r)} = c\cdot(kd + |G^{(r)}|\log k)$. Then, it follows from Theorem~5.7 of~\cite{AB99} that, for some universal constant $c > 0$, the first condition is violated at the $r$-th round with probability at most $\delta^{(r)}/6$. By a union bound over all possible $r$, the first condition holds with probability at least $1 - \sum_{r=1}^{+\infty}\delta^{(r)}/6 \ge 1 - \delta / 3$.
    
    Again, by Theorem~5.7 of~\cite{AB99}, the probability for the third condition to be violated for a specific $i$ is at most $\delta / (3k)$. Since $|G^{(r)}| \le k$, by a union bound, the third condition holds with probability at least $1 - k \cdot \frac{\delta}{3k} = 1 - \delta / 3$.

    Finally, a Chernoff bound shows that the second condition holds for a specific $r$ and $i \in G^{(r)}$ with probability at least
    \[
        1 - 2\exp\left(-\Omega\left(c \cdot \frac{\ln(|G^{(r)}| / \delta^{(r)})}{\eps} \cdot \eps\right)\right),
    \]
    which can be made greater than $1 - \frac{\delta^{(r)}}{6|G^{(r)}|}$ for sufficiently large $c$. By a union bound over all $r$ and $i \in G^{(r)}$, the second condition holds with probability at least
    \[
        1 - \sum_{r=1}^{+\infty}|G^{(r)}|\cdot\frac{\delta^{(r)}}{6|G^{(r)}|}
        \ge 1 - \delta / 3.
    \]
    The lemma follows from the three claims above and yet another union bound.
\end{proof}

\begin{lemma}\label{lemma:general-correctness}
    When event $\goodevent$ happens, the output of Algorithm~\ref{algo:general} satisfies $\err_{\D_i}(\hat f_i) \le 8\eps$ for every $i \in [n]$.
\end{lemma}

\begin{proof}
    We assign a classifier as $\hat f_i$ for some $i \in [n]$ either right after Line~\ref{line:general-test} or on Line~\ref{line:general-naive-ERM}. In either case, event $\goodevent$ guarantees that $\hat f_i$ is $8\eps$-accurate on $\D_i$.
\end{proof}

\begin{lemma}\label{lemma:general-efficiency}
    When event $\goodevent$ happens, Algorithm~\ref{algo:general} terminates with a sample complexity of
    \[
        O\left(\frac{kd\log(n/k)\log(1/\eps)}{\eps} + \frac{n\log k\log(1/\eps) + n\log(n/\delta)}{\eps}\right).
    \]
\end{lemma}

\begin{proof}
    We first control the size of $G^{(r)}$ in each round $r$. We claim that when event $\goodevent$ happens, if $G^{(r+1)}$ is defined during the execution of Algorithm~\ref{algo:general}, it holds that $|G^{(r+1)}| \le |G^{(r)}| / 2$. Indeed, the first condition of $\goodevent$ guarantees that for $\overline{\D} = \frac{1}{|G^{(r)}|}\sum_{i \in G^{(r)}}\D'_i$,
    \[
        \frac{1}{|G^{(r)}|}\sum_{i \in G^{(r)}}\err_{\D_i}\left(f_{c_i}\right)
        =   \frac{1}{|G^{(r)}|}\sum_{i \in G^{(r)}}\err_{\D'_i}\left(g_{f, c}\right) =   \err_{\overline{\D}}\left(g_{f,c}\right)
    \le 2\eps.
    \]
    By Markov's inequality, it holds for at least half of the values $i \in G^{(r)}$ that $\err_{\D_i}\left(f_{c_i}\right) \le 4\eps$. Then, the second condition of $\goodevent$ guarantees that $|G^{(r+1)}| \le |G^{(r)}| / 2$. It follows immediately that $|G^{(r)}| \le 2^{1-r}\cdot n$.

    Then, the sample complexity of the $r$-th iteration of the while-loop is upper bounded by
    \begin{align*}
        &~c\cdot\frac{d^{(r)}\ln(1/\eps) + \ln(1/\delta^{(r)})}{\eps} + |G^{(r)}|\cdot c\cdot \frac{\ln(|G^{(r)}| / \delta^{(r)})}{\eps}\\
    \preceq &~ \frac{(kd + 2^{-r}\cdot n\log k)\log(1/\eps) + \log(1/\delta) + \log r}{\eps} + 2^{-r}\cdot n\cdot\frac{\log(2^{-r}\cdot n) + \log(1/\delta) + \log r}{\eps}.
    \end{align*}
    Since the while-loop terminates whenever $|G^{(r)}| \le k$, there are at most $O(\log(n/k))$ iterations, and summing the above over $r = 1, 2, \ldots, O(\log(n/k))$ gives
    \begin{align*}
        &~\frac{kd\log(1/\eps) + \log(1/\delta)}{\eps}\cdot\log(n/k) + \frac{n\log k\log(1/\eps)}{\eps} + \frac{\log^2(n/k)}{\eps} + \frac{n\log n + n\log (1/\delta)}{\eps}\\
    \preceq &~\frac{kd\log(n/k)\log(1/\eps)}{\eps} + \frac{n\log k \log(1/\eps) + n\log(n/\delta)}{\eps}.
    \end{align*}
    Finally, the last for-loop of the algorithm takes $O\left(\frac{kd\log(1/\eps) + k\log(k/\delta)}{\eps}\right)$ samples in total, which is always dominated by the above. Therefore, we have the desired upper bound on the sample complexity.
\end{proof}

Finally, we put all the pieces together and prove Theorem~\ref{thm:sample-upper-general}.

\begin{proof}[Proof of Theorem~\ref{thm:sample-upper-general}]
    Lemmas \ref{lemma:general-good-event}, \ref{lemma:general-correctness}, and \ref{lemma:general-efficiency} together imply that, conditioning on an event that happens with probability at least $1 - \delta$, Algorithm~\ref{algo:general} returns $8\eps$-accurate classifiers for each of the $n$ distributions, while the number of samples is upper bounded by some
    \[
        M = O\left(\frac{kd\log(n/k)\log(1/\eps)}{\eps} + \frac{n\log k\log(1/\eps) + n\log(n/\delta)}{\eps}\right).
    \]
    To control the sample complexity---the (unconditional) expectation of the number of samples, we simply terminate the algorithm whenever the number of samples exceeds $M$. The resulting algorithm is still $(8\eps, \delta)$-PAC guarantee, and satisfies the desired upper bound on the sample complexity.
\end{proof}

\end{document}